\documentclass{article}

\usepackage[preprint]{neurips_2023}
\usepackage{natbib}
\usepackage{amsmath,amssymb,amsthm,mathtools,bm}
\usepackage{booktabs}
\usepackage{graphicx}
\usepackage[font=small,labelfont=bf]{caption}
\graphicspath{{figures/}}
\usepackage{enumitem}
\usepackage{microtype}
\usepackage{placeins}
\usepackage{xcolor}
\usepackage{xurl}
\usepackage{hyperref}

\definecolor{paperlink}{RGB}{35,65,100}
\hypersetup{colorlinks=true,linkcolor=paperlink,citecolor=paperlink,urlcolor=paperlink,
  pdftitle={Learning to Extrapolate in Depth with Looped Transformers},
  pdfauthor={Huitao Yang, Jianhao Huang, Zixuan Wang, Baharan Mirzasoleiman}}
\newtheorem{theorem}{Theorem}
\newtheorem{proposition}{Proposition}
\newtheorem{lemma}{Lemma}

\newtheorem{assumption}{Assumption}
\newtheorem{definition}{Definition}
\newtheorem{remark}{Remark}
\newtheorem{example}{Example}

\begin{document}
\raggedbottom

\title{Learning to Extrapolate in Depth\\with Looped Transformers}

\author{%
  {\normalsize
  Huitao Yang$^{1}$\thanks{Equal contribution.} \quad
  Jianhao Huang$^{1}$\footnotemark[1] \quad
  Zixuan Wang$^{2}$ \quad Baharan Mirzasoleiman$^{1}$}\\
  $^{1}$University of California, Los Angeles\\
  $^{2}$Princeton University
}
\maketitle

\begin{abstract}
    Looped transformers reuse a single block across iterations, raising a natural question: \textit{how far can a block trained at one depth be safely extrapolated to greater depths?} We study this question for all-pairs graph connectivity, where harder graphs require longer-range information propagation. We show that under exchangeability and symmetric positive initialization, population training at depth $K$ converges to an update rule governed by a single effective parameter that balances graph propagation against fitting structure unresolved at training depth. This trade-off creates a finite \textit{extrapolation window}: too few test-time loops provide insufficient propagation, while too many eventually create false connections. We derive explicit endpoints for this window and bound the graph hardness solvable beyond training as a function of training depth and hardness. In particular, a test depth chosen solely from the training distribution can guarantee recovery on substantially harder graphs, with a certified depth gain that grows exponentially as training depth outpaces training hardness. Experiments validate the predicted windows and show how compute budget and initialization affect the optimal training depth.
\end{abstract}

\section{Introduction}
In standard transformers, depth and parameter count grow together. Looped
transformers decouple the two by applying one shared block repeatedly
\citep{dehghani2019universal,giannou2023looped,yang2024looped}. Because the
number of loops can be changed after training, looping offers a simple way
to scale test-time compute: recurrent-depth language models improve when
unrolled further at inference \citep{geiping2025scaling}, and a few looped
layers can match much deeper non-looped models \citep{saunshi2025reasoning}.
This raises a natural question: \textit{how far can a looped transformer safely extrapolate beyond the depth at which it was trained?} A natural approach is to train with a few loops on tractable instances, then tackle harder ones by running the block for more loops \citep{schwarzschild2021can,fan2025looped}.
Yet more computation does not always help. Performance can deteriorate
under excessive test-time iteration, a phenomenon often described as
overthinking \citep{bansal2022end}, and transformers often fail on inputs
longer than those seen in training \citep{anil2022exploring,zhou2024algorithms}.
What, then, does training actually learn at a fixed depth, and what happens
when that learned computation is run longer? 

Graph connectivity provides a controlled setting to answer these questions
\citep{velickovic2020neural,velickovic2022clrs,ye2026connectivity}. Deciding
whether two vertices are connected requires composing local edges into
global reachability, so harder instances require information to propagate
over longer distances. Instance hardness therefore translates naturally
into the number of iterations needed for successful computation.

Theory has made progress on related questions, often using graph reasoning
as a testbed. On the expressivity side, \citet{giannou2023looped} and
\citet{deluca2024simulation} constructed looped transformers that execute
general programs and graph algorithms, while
\citet{sanford2024understanding} and \citet{merrill2025little} showed that
logarithmic depth suffices for graph connectivity. On the optimization
side, \citet{gatmiry2024looped} proved that linear looped transformers learn
multi-step preconditioned gradient descent for in-context regression, and
\citet{wu2026power} proved that a looped linear transformer with layer
normalization learns the power method, so that unrolling the trained block
runs further power iterations. For graph connectivity itself,
\citet{ye2026connectivity} connected these perspectives: a simplified
$L$-layer transformer can represent connectivity for graphs of diameter up
to $3^L$, and training finds this solution only when most training graphs
lie within that capacity. 

These results establish that looped transformers can represent iterative
algorithms and, in some settings, that training can discover them. They do
not, however, characterize what happens when a trained block is frozen and
unrolled beyond its training depth. In particular, training at a fixed
depth $K$ need not recover an ideal iterative rule
\citep{ye2026connectivity}. Instead, the learned block balances propagating
graph information against fitting the structure left unresolved at depth
$K$. This creates a fundamental tension when the block is reused:
additional loops initially propagate information farther and enable harder
instances, but repeated application also compounds the learned bias and can
eventually create false connections. Thus, rather than improving
monotonically with test-time depth, a trained looped transformer can have a
finite \textit{extrapolation window} in which additional computation helps.

We study this phenomenon on all-pairs graph connectivity. A looped attention
block is trained at depth $K$ and then evaluated with frozen weights at
depths $T>K$ on harder test graphs, requiring exact recovery of all
labels. Figure~\ref{fig:depth-overview} illustrates the phenomenon we
explain: a block trained with $K=3$ loops on four-vertex paths recovers
16-vertex paths exactly at $T=6$, while fewer or more loops fail. We ask
when such an extrapolation window exists, how its boundaries depend on
training, and how much harder a graph can be solved beyond the training
regime. Our main contributions are as follows:

\begin{itemize}[leftmargin=*,itemsep=2pt,topsep=2pt]
\item We characterize the training dynamics. For exchangeable training
distributions and symmetric positive initialization, population gradient
flow on the full weight matrices reduces to a single effective parameter
and converges exponentially to an explicit optimum
(Theorem~\ref{thm:gf}). This optimum balances propagating graph information
against the residual left unresolved at the training depth.

\item We characterize the resulting extrapolation window. The trained block
recovers every connectivity label of a test graph throughout a finite
window of test depths (Proposition~\ref{prop:test-window}). The lower end
reflects the graph's hardness, while the upper end marks when additional
loops begin to create false connections. An Experiments on Erdős–Rényi (ER) graphs illustrate these finite windows of successful test depths
(Figure~\ref{fig:ood-window-families}).

\item We quantify how far extrapolation can go. We prove that the learned block can solve every test
graph up to a hardness upper-bound depending on training difficulty and training depth.
(Theorem~\ref{thm:hardness-gap}). A high-probability version extends this
result to training distributions with tails of harder graphs
(Theorem~\ref{thm:probabilistic-extrapolation}).

\item We instantiate these guarantees under two distribution shifts, from
short to longer paths and from sparse to denser Erd\H{o}s--R'enyi graphs
(Examples~\ref{ex:path-generalization} and \ref{ex:er-generalization}).
Numerical experiments evaluate transfer to harder graphs
(Figure~\ref{fig:ood-windows}) and show how the choice of training depth
depends on compute and initialization (Section~\ref{sec:finite-compute}).
\end{itemize}

To the best of our knowledge, these are the first guarantees that quantify
how far a trained looped transformer can extrapolate beyond its training
depth and when additional test-time iterations become harmful. More broadly,
our analysis explains why test-time looping can first enable harder
computation and then fail: the same learned update rule that propagates
information farther also determines the eventual limit of safe
extrapolation.

\section{Related Work}
\label{sec:related}

\paragraph{Looped transformers and test-time depth.}
Parameter sharing across depth goes back to Universal Transformers
\citep{dehghani2019universal} and ALBERT \citep{lan2020albert}. Looped
transformers can be programmed as general-purpose computers
\citep{giannou2023looped}, learn in-context algorithms with few parameters
\citep{yang2024looped}, match deeper non-looped models on reasoning tasks
\citep{saunshi2025reasoning}, and, at language-model scale, improve when
unrolled further at test time \citep{geiping2025scaling}. Recurrent models
can also solve harder instances by allocating more computation
\citep{schwarzschild2021can}, but performance can degrade when they are
iterated too long \citep{bansal2022end}; looped transformers with
input-adaptive step counts have also shown length generalization on
algorithmic tasks \citep{fan2025looped}. These works establish the empirical
promise and limitations of additional computation; we instead characterize
theoretically what a trained looped block does when it is frozen and
unrolled beyond its training depth.

\paragraph{Expressivity for algorithmic and graph reasoning.}
Transformers with hard attention are Turing complete
\citep{perez2021attention}, and looped transformers simulate graph
algorithms with a graph-size-independent parameter count
\citep{deluca2024simulation}. Logarithmic depth is sufficient for graph
connectivity \citep{sanford2024transformers,merrill2025little} and, conditionally on a computational conjecture, necessary for sublinear-width models \citep{sanford2024understanding}. 
Our looped architecture achieves a $3^L$ connectivity range after $L$ loops, consistent with this logarithmic-depth scaling. These expressivity results establish what a looped transformer can represent, but do not characterize the behavior of a block after it has been trained at a fixed depth and subsequently unrolled beyond that depth.

\paragraph{Training dynamics of algorithm-implementing transformers.}
For in-context linear regression, trained transformers were observed to fit
simple function classes \citep{garg2022what} and to mimic gradient descent
\citep{vonoswald2023transformers}; single-layer linear attention provably
converges to one step of preconditioned gradient descent
\citep{ahn2023transformers,zhang2024trained}. Linear looped transformers
learn multi-step preconditioned gradient descent
\citep{gatmiry2024looped}, with polynomially many in-context examples
\citep{chen2025bypassing}, and with layer normalization they learn the
power method \citep{wu2026power}. On graph connectivity,
\citet{ye2026connectivity} showed that a simplified transformer learns the
matrix-power algorithm only when most training graphs are within its
capacity. As in these works, symmetry reduces our full matrix dynamics to
a few parameters; unlike them, our main object is the trained block run
beyond its training depth on harder test inputs.

\paragraph{Length and size generalization.}
Transformers often fail on inputs longer than those seen in training
\citep{anil2022exploring}, and are conjectured to length-generalize when
the task has a short RASP-L program \citep{zhou2024algorithms}. For graph
neural networks, extrapolation succeeds when the architecture encodes the
target's nonlinearity \citep{xu2021neural} and can fail when local
structure changes with graph size \citep{yehudai2021local}; neural
algorithmic reasoning evaluates learned graph algorithms on larger graphs
\citep{velickovic2020neural,velickovic2022clrs}. In our setting the
architecture is fixed and extrapolation comes only from extra loops; we
certify how much harder, in component size and spectral gap, a test graph
can be as a function of the training depth and distribution.

\section{Problem Setting}
\label{sec:setting}

\begin{figure}[t]
\centering
\includegraphics[width=\textwidth]{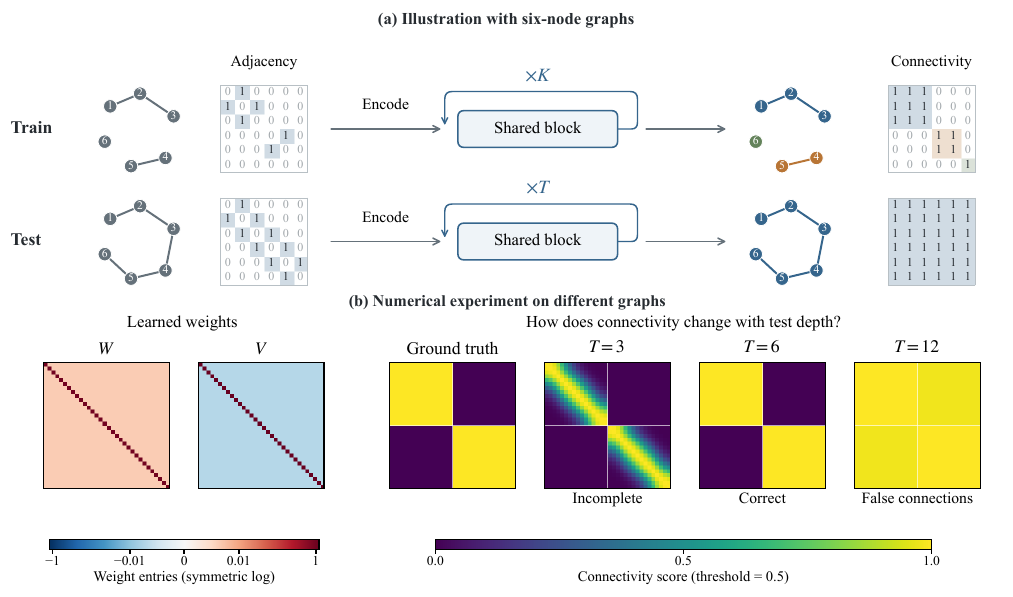}
\caption{\textbf{Depth extrapolation on graph connectivity.} \textbf{Top:} Training and inference with a shared block. \textbf{Bottom:} Learned model weights and test-time predictions across iterations. We train a shared block ($K=3$ loops) on eight disjoint 4-vertex paths and evaluate extrapolation performance on two disjoint 16-vertex paths. $T=6$ test loops perfectly recovers all test ground-truth labels, whereas insufficient or excess test iterations lead to under- or over-prediction of pairwise connectivity.}
\label{fig:depth-overview}
\end{figure}

We study whether a looped transformer trained with $K$ loops can solve
harder graphs by using more loops at test time. We train one shared
attention block on graph connectivity and keep its weights fixed during
inference. This task lets us relate the number of loops to the amount of
information propagation required for exact recovery
\citep{ye2026connectivity}.

\paragraph{Graph connectivity task.}
Let $\mathcal G_N$ be the set of simple, undirected graphs on the
labeled vertex set $[N]$. Given $G\in\mathcal G_N$, the model predicts
the binary matrix $Y_G\in\{0,1\}^{N\times N}$, where
$(Y_G)_{ij}=1$ if and only if vertices $i$ and $j$ are connected by
a path, including the zero-length path when $i=j$.
For example, edges $(1,2)$ and $(2,3)$ imply the positive label
$(Y_G)_{13}=1$, even without edge $(1,3)$. Pairs in different
components have label zero. 

\paragraph{Input representation and supervision.}
Let $A_G$ be the adjacency matrix,
$L_G=\operatorname{diag}(A_G\mathbf1)-A_G$ its Laplacian, and
$d_{\max}(G)$ the maximum degree. A fixed input map produces
$N$ node tokens of width $N$:
\begin{equation}
\label{eq:input}
H_0=P_G:=I-\frac{L_G}{2\max\{1,d_{\max}(G)\}}.
\end{equation}
Row $i$ encodes vertex $i$ and its neighbors; there is no learned
input embedding. If $C(i)$ is the component containing $i$, the
supervised target is $(\Pi_G)_{ij}=(Y_G)_{ij}/|C(i)|$.
Thus, $\Pi_G$ has the same support as $Y_G$ and
$P_G^m\to\Pi_G$ (Appendix~\ref{app:proof:prop:spectrum}).
We use $P_G$ as the input and $\Pi_G$ as the training target.

\paragraph{Looped self-attention model.}
The hidden state $H_t\in\mathbb R^{N\times N}$ evolves through
one attention block with parameters $\theta=(W,V)$ shared across
all loops:
\begin{equation}
\label{eq:model}
H_{t+1}=\Phi_\theta(H_t)
:=\mathcal N\bigl((H_tWH_t^\top)H_tV\bigr).
\end{equation}
Here $W,V\in\mathbb R^{N\times N}$ parameterize the bilinear
attention scores and the value projection. We use linear attention
without softmax, a residual connection, or a feed-forward sublayer.
The map $\mathcal N(Z)=\operatorname{diag}(Z\mathbf1)^{-1}Z$
normalizes each row of the block output to sum to one. The positive
parameter regime below ensures that these
row sums are nonzero. For $G\sim\mathcal D_{\rm tr}$, training
uses $K\ge1$ loops and minimizes the terminal population loss
$\mathcal L_K(\theta)=\tfrac12\mathbb E_G\|H_K-\Pi_G\|_F^2$;
intermediate hidden states receive no separate supervision.

\paragraph{Inference and depth generalization.}
At test time, we freeze $\theta$ and run the block for
$T\ge0$ loops. A fixed readout with threshold $\tau\in(0,1)$
produces
\begin{equation}
\label{eq:decoder}
\widehat Y_{ij}(T)=\mathbf1\!\left\{
\frac{(H_T)_{ij}}{\sqrt{(H_T)_{ii}(H_T)_{jj}}}>\tau
\right\}.
\end{equation}
Success means \emph{exact match}, $\widehat Y(T)=Y_G$, on all
vertex pairs. 
We normalize the output before thresholding so that the model remains scale-invariant to component size, offsetting the $1/n_C$ target scaling assigned to node pairs in a component of size $n_C$ during training.
For each problem instance, the input and target use the same $N$ vertices, while $\mathcal D_{\rm te}$ may differ from
$\mathcal D_{\rm tr}$ in graph size, component size, and connectivity structure.
Depth extrapolation uses $T>K$ to solve harder test instances with
no additional parameters or test-time optimization.
Figure~\ref{fig:depth-overview} shows this training and inference
procedure, the learned matrices, and the unthresholded readout on
longer test paths.

\section{Training Dynamics}
\label{sec:training-dynamics}

\paragraph{One effective parameter.}
Let $E=\mathbf 1\mathbf 1^\top/N$ and $F=I-E$. For weights
$W=aF+bE$ and $V=cF+dE$ with positive coefficients, the normalized
output depends on the weights only through the ratio $\gamma=ac/(bd)$:
\begin{equation}
\label{eq:forward}
H_t=E+\gamma^{(3^t-1)/2}\bigl(P_G^{3^t}-E\bigr).
\end{equation}
Each loop cubes the graph operator, tripling the number of graph
propagation steps. The factor $\gamma$ controls the magnitude of the graph-dependent component relative to the uniform matrix $E$.
When $\gamma=1$, the model computes $P_G^{3^t}$. When $\gamma<1$,
looping also moves the output toward $E$.
Under the symmetry and initialization conditions below, population
training preserves this weight structure
(Appendix~\ref{app:setting-training}).

Write $c_G$ for the number of components, and define
$\kappa=\mathbb E_G(c_G-1)$ and
$\delta_K=\mathbb E_G\|P_G^{3^K}-\Pi_G\|_F^2$.
These measure, respectively, the expected number of components beyond one in the training distribution
and the error remaining after $3^K$ graph steps. The loss is a
quadratic in $\gamma^{(3^K-1)/2}$, so training selects a single
optimal ratio.

\begin{theorem}[Population training learns the effective ratio]
\label{thm:gf}
Suppose $G\sim\mathcal D_{\rm tr}$ implies
$\pi(G)\sim\mathcal D_{\rm tr}$ for every permutation $\pi$
of $[N]$, and $\kappa>0$. Initialize $W_0=a_0F+b_0E$ and $V_0=c_0F+d_0E$
with $a_0,b_0,c_0,d_0>0$. Population matrix gradient flow converges
to a positive stationary pair with
\begin{equation}
\label{eq:opt}
\gamma_*=
\left(1+\frac{\delta_K}{\kappa}\right)^{-2/(3^K-1)}.
\end{equation}
For fixed $K$, there are constants $\mu>0$ and $C\ge0$ such that,
at optimization time $r\ge0$,
\begin{equation}
\label{eq:gf-rates}
\begin{aligned}
\mathcal L_K(r)-\mathcal L_*
&\le e^{-2\mu r}\bigl(\mathcal L_K(0)-\mathcal L_*\bigr),\\
|\gamma(r)-\gamma_*|&\le Ce^{-\mu r},
\end{aligned}
\end{equation}
where $\mathcal L_*=\kappa\delta_K/[2(\kappa+\delta_K)]$ is the
minimum over this positive parameter family.
\end{theorem}

At the training depth, the coefficient on the graph-dependent component is
$\gamma_*^{e_K}=\kappa/(\kappa+\delta_K)$, 
where $e_K=(3^K-1)/2$. 
The graph-dependent component distinguishes connected components, but also carries the propagation error
$P_G^{3^K}-\Pi_G$. 
Thus, training balances preserving graph structure against suppressing the propagation error $P_G^{3^K}-\Pi_G$.
As the number of training loops grows, the propagation error
$\delta_K$ decreases and $\gamma_*$ approaches one
(Appendix~\ref{app:proof:prop:monotonicity}). The learned block therefore attenuates the graph-dependent component more slowly over subsequent test loops.

\paragraph{Why training converges.}
Relabeling the vertices leaves the population loss unchanged, and symmetric initialization therefore keeps the gradient flow in $\operatorname{span}\{E,F\}$.
Within this family, the loss is a convex quadratic in
$q=\gamma^{e_K}$, with minimizer $q_*=\kappa/(\kappa+\delta_K)$.
A gradient-domination bound gives the exponential convergence rate
(Appendix~\ref{app:proof:thm:gf}). Gradient descent reaches the same
limit with a sufficiently small constant learning rate
(Appendix~\ref{app:proof:thm:gd}).
Figure~\ref{fig:population-convergence} illustrates this convergence
from four positive initializations.

\paragraph{The role of training depth.}
More training loops reduce the error at convergence, but the time needed
to reach this solution depends on initialization. 
When the effective ratio is below one, i.e.,
$\gamma_0<1$ (or once $\gamma$ has progressed toward $\gamma^* < 1$ enough), 
the factor $\gamma_0^{e_K}$ becomes very small as $K$
grows and suppresses the gradient and slowing optimization. 
Thus, deeper training improves the asymptotic solution but can require substantially more computation to reach it.
Section~\ref{sec:finite-compute} brings this optimization effect
together with the extrapolation guarantees to discuss how to choose $K$.

\begin{figure}[t]
\centering
\includegraphics[width=0.85\textwidth]{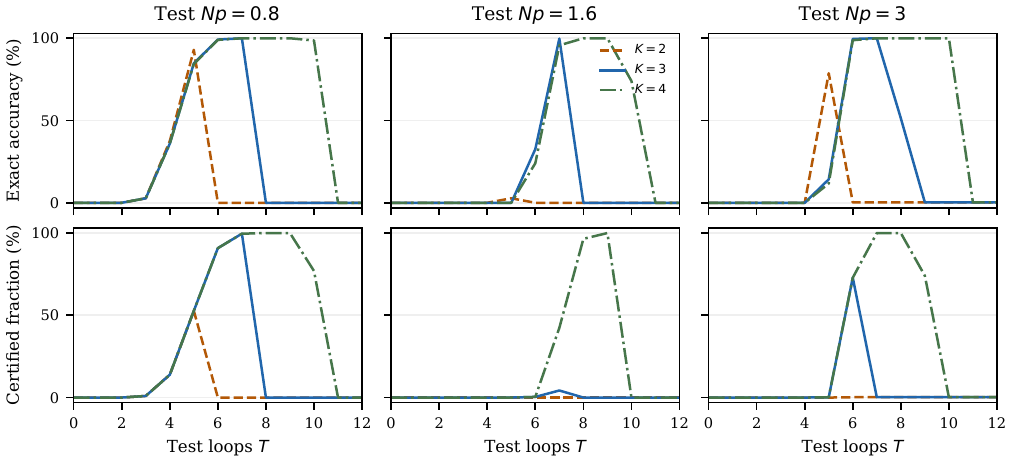}
\caption{\textbf{Test-depth window at population optimum.}
We train on 8192 ER graphs with $N=128$ and $Np_{\rm tr}=0.8$
at depths $K=2,3,4$, and evaluate at varying test depths and densities.
\textbf{Top:} exact graph accuracy averaged over test data. \textbf{Bottom:} the fraction of test graphs
whose certified test-depth window in Proposition~\ref{prop:test-window} contains each
$T$. The certified fraction lower-bounds exact accuracy on the same test samples.}
\label{fig:ood-window-families}
\end{figure}

\section{Extrapolation}
\label{sec:extrapolation}

We now investigate test-time scaling via looping with the model obtained at
the population optimum. We first characterize a range of test depths that
gives exact recovery and show how training depth determines which harder
graphs this range covers (Section~\ref{sec:test-depth}). We then derive
asymptotic generalization guarantees and instantiate them on two common
graph families (Section~\ref{sec:training-depth}). Finally, we provide
empirical insight into how to choose the training loop count for better
generalization (Section~\ref{sec:finite-compute}).

\subsection{Extra loops and harder test graphs}
\label{sec:test-depth}

We first characterize, for a fixed test graph, which test depths are
guaranteed to recover its connectivity.

\begin{proposition}[Test-depth window for a fixed graph]
\label{prop:test-window}
At the population optimum of Theorem~\ref{thm:gf}, let $G$
be a test graph with largest
component size $s_{\max}$ and largest non-unit eigenvalue
$\rho\in(0,1)$ of $P_G$. The learned block recovers every connectivity
label of $G$ at every integer $T\ge0$ satisfying
\begin{equation}
\label{eq:test-window-raw}
\begin{aligned}
\log_3\!\left(
\frac{
\log[s_{\max}(1+\tau)/(1-\tau)]
}{
-\log\rho
}
\right)&<T,\\
T&<T_{\max}(G;K),
\end{aligned}
\end{equation}
where
\begin{equation}
\label{eq:test-upper}
\begin{aligned}
T_{\max}(G;K):=\log_3\!\Bigg[1
&+\frac{3^K-1}{\log(1+\delta_K/\kappa)}\\
&\!\cdot
\log\!\left(
1+\frac{\tau N}{(1-\tau)s_{\max}}
\right)
\Bigg].
\end{aligned}
\end{equation}
We set $T_{\max}=+\infty$ when $\delta_K=0$. For a connected test
graph, the upper restriction is unnecessary.
\end{proposition}


\emph{Proof Idea.} The two endpoints arise from the two types of
vertex pairs. For connected pairs, sufficient propagation makes the
normalized readout exceed $\tau$, yielding the lower endpoint. This
endpoint depends only on the test graph and readout threshold and
measures how many loops are needed to propagate connectivity, which
motivates our definition of graph hardness below. For disconnected
pairs, the accumulated uniform contribution eventually pushes the
readout above $\tau$ when $\delta_K>0$, yielding the upper endpoint.
Thus, excessive iteration can create false-positive connections, and
the upper endpoint reflects the bias learned during training (Appendix~\ref{app:proof:prop:test-window}).

Figure~\ref{fig:ood-window-families} provides an empirical illustration. We trained looped transformers on 8192 graphs sampled from the same Erd\H{o}s--R\'enyi (ER) distribution $G(128, 0.8)$ (Definition see Example~\ref{ex:er-generalization}) and evaluated on 512 independent test graphs sampled from different ER distributions while varying test loop counts. We compare the empirically observed successful test-depth windows with the theorem's certified windows, finding that the latter are conservative and therefore provide lower bounds on the true successful range. 


\begin{definition}[Graph hardness]
\label{def:hardness}
For a graph $G$ with largest component size $s_{\max}$ and largest
non-unit eigenvalue $\rho\in(0,1)$, define
\begin{equation}
\label{eq:hardness}
\mathsf{Hard}(G)
:=
\log_3\!\left(
\frac{
\log[s_{\max}(1+\tau)/(1-\tau)]
}{
-\log\rho
}
\right).
\end{equation}
Set $\mathsf{Hard}(G)=-\infty$ when $P_G=\Pi_G$.
\end{definition}

The hardness has a direct graph interpretation. A larger
$s_{\max}$ means that connectivity may need to be propagated across a
larger component, while $\rho$ controls how quickly this information
spreads. Matrix powering expands the effective graph distance seen by
the model: path-like or poorly connected components can have $\rho$ close to one and therefore require many loops, whereas in the limiting case $\rho=0$ where all components are cliques, the
input representation already reveals the connected components and no
extra loop is needed.

With this notation, Proposition~\ref{prop:test-window} becomes simply
\begin{equation}
\label{eq:test-window}
\mathsf{Hard}(G)
<
T
<
T_{\max}(G;K).
\end{equation}

Figure~\ref{fig:ood-window-families} compares test-time looping after
training with different loop counts on the same ER population.
For each $K=2,3,4$, we freeze the trained weights and vary the test-loop
count $T$. Accuracy first rises as extra loops
propagate information through each component, then falls when further
loops create false connections. More training loops widen the useful
range of test depths. The lower row shows the fraction of test graphs
covered by the certified test-depth windows. Appendix~\ref{app:er-protocol}
gives the protocol and a longer-depth comparison.

\subsection{Generalization from shallow to deep tasks}
\label{sec:training-depth}

We next ask how much harder a test graph can be than the training
graphs. We first show that, under a spectral regularity assumption on the training
distribution, the trained looped transformer can provably generalize to
harder instances in asymptotic sense
(Appendix~\ref{app:uniform-extrapolation}).

\begin{theorem}[Generalization from shallow to deep graphs]
\label{thm:hardness-gap}
Fix $\tau\in(0,1)$. Under the assumptions of
Theorem~\ref{thm:gf}, consider a training family satisfying the spectral regularity condition
defined in Appendix~\ref{app:hardness-gap-proof}, with maximum training
hardness \(H\) and training depth \(K\).
If
\[
H\to\infty,
\qquad
H3^{K-H}\to\infty,
\]
then there exists a single integer inference depth $T>K$, determined
only by the training problem, such that
\begin{equation}
\label{eq:hardness-scaling}
T-K=\Theta\!\left(H3^{K-H}\right),
\end{equation}
and the population-optimal block solves every test graph on the same
$N$ vertices satisfying
\[
\mathsf{Hard}(G)<T.
\]
The constants in the asymptotic rate are uniform over the family.
\end{theorem}

The gain is controlled by the gap $K-H$ between training depth and
training hardness. Spectral regularity means, informally, that the eigenmode of $P_G$ with the slowest decay under repeated powering determines both the training hardness and the residual error after $K$ loops; the formal condition is given in Appendix~\ref{app:hardness-gap-proof}. Consequently, the number of
additional inference loops guaranteed beyond the training depth grows
as $\Theta(H3^{K-H})$. In particular, when $K-H=O(1)$, the guaranteed additional depth is $\Theta(H)$. We then provide an example to show this gain in a specific graph family.

\begin{figure}[!t]
\centering
\includegraphics[width=\textwidth]{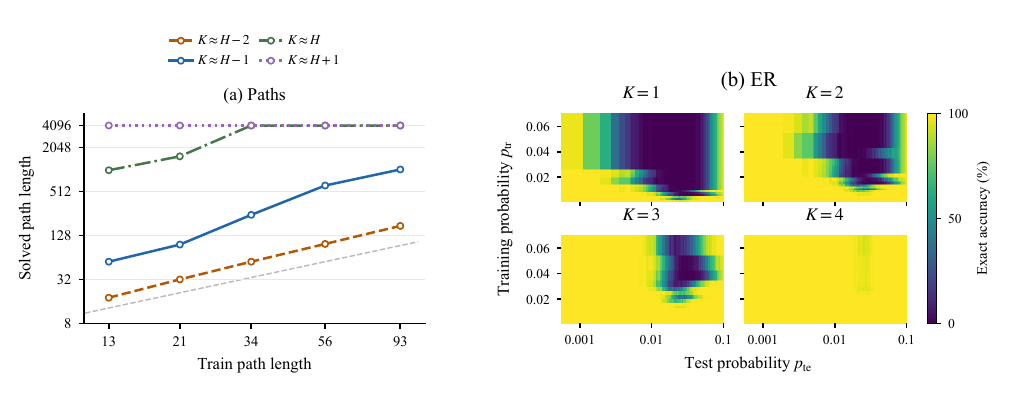}
\caption{\textbf{Transfer to longer paths and denser ER graphs.}
We train at the population optimum on disjoint short paths
or ER graphs, and evaluate transfer to longer paths or denser ER graphs.
\textbf{(a)} Solved path length at fixed $N=8192$, with training loops $K$ tracking
training hardness $H$. Both axes use log scales.
\textbf{(b)} Maximum exact accuracy on 128 ER test graphs per density at
$N=64$, using a common $T\in\{0,\ldots,18\}$ in each cell.}
\label{fig:ood-windows}
\end{figure}

\begin{example}[From short paths to long paths]
\label{ex:path-generalization}
We denote by $P_L$ a path with $L$ vertices. For $N\ge2s$, train
at the population optimum on uniform relabelings of a graph with
$\lfloor N/s\rfloor$ copies of $P_s$ and isolated leftover vertices,
under the initialization assumptions of Theorem~\ref{thm:gf}.
Test graphs have $\lfloor N/L\rfloor$ copies of $P_L$ and isolated
leftover vertices, on the same $N$ vertices.

As $s\to\infty$, suppose
$K-\mathsf{Hard}_{\rm tr}\to\Delta$ for a constant $\Delta$. By
the refinement of Theorem~\ref{thm:hardness-gap} proved in
Appendix~\ref{app:spectral-refinement}, there exists one test depth
$T>K$, depending only on the training problem, at which the trained
block solves every feasible test size $L\le N$ whenever
$$
L=s^{q+o(1)}, \qquad 1<q<1+3^\Delta.
$$
Thus, a constant gap $K-\mathsf{Hard}_{\rm tr}\to\Delta$ permits
extrapolation from training paths of length $s$ to paths whose length
grows as $s^{q+o(1)}$ for $q<1+3^\Delta$.

Empirically, we find that solvability is monotone in path length:
whenever a model solves $G_L$, it also solves all tested graphs with
shorter paths. In our experiments, the solved set therefore forms a prefix
$L=2,\ldots,L'$. We thus report the hardness of $L'$ as the measured
test-depth cutoff for each training path length.
\end{example}

In addition, we show that training distributions with a tail of harder instances can admit a
related one-sided guarantee, thus making extrapolation gain applicable for several random graph distributions.

\begin{theorem}[High-probability shallow-to-deep generalization]
\label{thm:probabilistic-extrapolation}
Under the assumptions of Theorem~\ref{thm:gf}, fix
$\tau\in(0,1)$ and suppose
\[
\Pr_{\mathcal D_{\rm tr}}
\bigl(\mathsf{Hard}(G)>H\bigr)
\le\varepsilon.
\]
Let
\[
M=
\min\left\{
3^{K-H},
\log\frac1\varepsilon
\right\}.
\]
If $M\to\infty$ and $\log(N/\kappa)=o(M)$, then there exists a
single integer inference depth $T>K$, independent of the test graph,
such that
\[
T-K=\Theta(M),
\]
and the population-optimal block solves every test graph satisfying
\[
\mathsf{Hard}(G)<T.
\]
\end{theorem}

The two terms in $M$ capture two distinct limitations on extrapolation.
When the training depth exceeds the training hardness, the residual
after training decreases rapidly, allowing the learned block to remain
useful for more additional loops; this yields the factor
$3^{K-H}$. On the other hand, if harder training instances occur with
probability at most $\varepsilon$, their tail probability limits how
far a common test depth can be guaranteed, yielding the factor
$\log(1/\varepsilon)$. The extrapolation gain is therefore controlled
by the smaller of these two quantities. The proof is given in
Appendix~\ref{app:probabilistic-extrapolation}. The next example applies
the same argument to Erdős--Rényi graphs.

\begin{example}[From sparse to denser ER graphs]
\label{ex:er-generalization}
We denote $G(n,p)$ for the Erd\H{o}s--R\'enyi distribution on $n$
labeled vertices, where each possible edge is included
independently with probability $p$.
Fix $\tau,\varepsilon\in(0,1)$ and train to the population optimum
on $G(n,p_{\rm tr})$, where $n\ge3$ and $0<np_{\rm tr}<1$,
under the assumptions of Theorem~\ref{thm:gf}. The finite-size refinement
of Theorem~\ref{thm:probabilistic-extrapolation} in
Appendix~\ref{app:er-generalization} constructs a training-selected depth
$T_K$ and a component-size budget $S_K$. Provided $T_K>K$ and
$S_K>\log(n/\varepsilon)/\alpha(np_{\rm tr})$, the model has
exact-match accuracy at least
$1-\varepsilon$ on every ER test distribution in the range
\begin{equation}
\label{eq:er-example-p-range}
p_{\rm tr}<p_{\rm te}\le\frac1n\alpha^{-1}\!\left(
\frac{\log(n/\varepsilon)}{S_K}\right),
\end{equation}
where $\alpha(c)=c-1-\log c$ and its inverse is on $(0,1)$, yielding generalization to denser graphs. $S_K$ is an integer constant determined by the training distribution,
$K$, $\tau$, and $\varepsilon$, defined in
Appendix~\ref{app:er-generalization}.
Empirically, we train looped transformers on ER graphs drawn from \(G(64,p_{\rm tr})\) for several training probabilities \(p_{\rm tr}\) and evaluate them on independent graphs from \(G(64,p_{\rm te})\) when looping for $T$ times. Figure~\ref{fig:ood-windows}(b) reports the maximum exact accuracy over preselected $T$ grid for each pair \((p_{\rm tr},p_{\rm te})\).
\end{example}

\subsection{How many training loops?}
\label{sec:finite-compute}

More training loops improve the population-optimal extrapolation
guarantee, but need not improve generalization under a fixed training
budget. The practical question is therefore not simply how large $K$
should be, but whether the corresponding iterative rule can be learned
within the available compute.

\paragraph{Deeper unrolling can become harder to optimize.}
Increasing $K$ raises both the cost of each update and the difficulty
of making optimization progress. Figure~\ref{fig:flop-matched} shows
this effect under matched compute. From the same initialization
$\gamma_0\approx0.96$, models with $K\le4$ reach the population optimum,
whereas $K=6,8$ remain near their initial loss despite separately tuned
learning rates. Simply increasing the learning rate does not solve this problem:
Figure~\ref{fig:lr-sweep} shows that they either remain on the plateau
or become unstable.

Thus, although increasing $K$ improves the optimal extrapolation window,
that improvement may be inaccessible under a finite training budget
when deeper unrolling suppresses the training signal.

\begin{figure}[t]
\centering
\begin{minipage}[t]{0.48\linewidth}
\vspace{0pt}
\centering
\begin{minipage}[c][0.74\linewidth][c]{\linewidth}
\centering
\includegraphics[width=\linewidth]{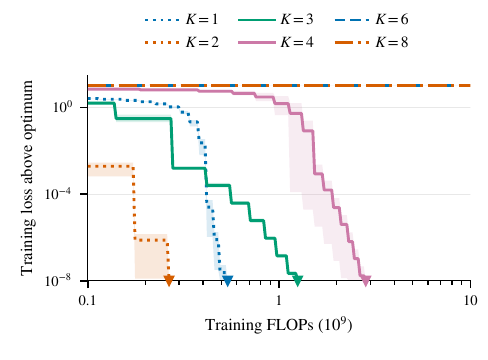}
\end{minipage}
\caption{\textbf{Deeper training can be harder to optimize at matched compute.}
Excess population loss versus equivalent training FLOPs for different
training depths $K$. We use population GD on mixtures of disjoint 4- and 8-vertex
paths with $N=128$, comparing training depths under matched compute. Learning rates are tuned separately for each $K$;
from $\gamma_0\approx0.96$, $K\le4$ converges while $K=6,8$ remain
near the initial plateau.}
\label{fig:flop-matched}
\end{minipage}\hfill
\begin{minipage}[t]{0.48\linewidth}
\vspace{0pt}
\centering
\begin{minipage}[c][0.74\linewidth][c]{\linewidth}
\centering
\includegraphics[width=\linewidth]{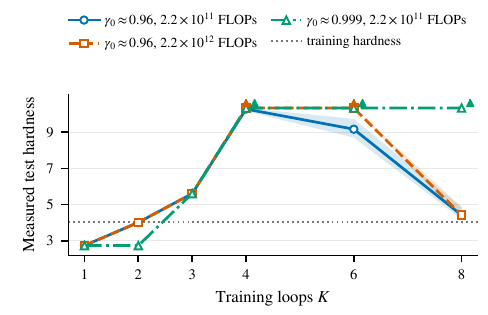}
\end{minipage}
\caption{\textbf{The best training depth depends on compute and initialization.} 
Hardest test-graph hardness reached under matched training budgets.
We train models on disjoint 4- and 8-vertex
paths with $N=128$, and evaluate on path, broom, and cycle graphs.
Moderate $K$ performs best from $\gamma_0\approx0.96$, while
near-identity initialization makes deeper training effective.
}
\label{fig:choose-k}
\end{minipage}
\end{figure}

\paragraph{The best training depth depends on compute and initialization.}
We next evaluate the harder graphs reached after Adam training
under matched budgets. At the smaller budget, $K=4$ gives the strongest
generalization from $\gamma_0\approx0.96$. With ten times more compute,
both $K=4$ and $K=6$ solve the full test catalog, while $K=8$ still
fails to improve over its initialization. More compute therefore
expands the useful range of $K$, but does not automatically make the
deepest training run preferable.

Initialization can change this picture substantially. Starting closer
to $\gamma=1$ allows $K=4,6,8$ to solve the full catalog at the smaller
budget. This agrees with the reduced dynamics: preserving more graph
information initially avoids the severe gradient suppression that can
occur at large $K$. At the same time, near-identity initialization is
already a stronger untrained baseline, so improved test performance
should not be attributed solely to faster optimization.

Therefore, training depth should be chosen jointly with initialization
and compute. While our theory characterizes the extrapolation benefit
after convergence, experiments show that larger $K$ is useful only when
optimization can reach that regime. In practice, candidate loop counts
should be compared at matched compute and selected using harder-task
validation performance rather than population-optimal guarantees alone.

\section{Conclusion \& Limitations}
\label{sec:conclusion}

We studied how training determines the benefit of extra test-time loops
for graph connectivity. Population training learns an effective ratio
that balances component separation against propagation error, creating
a finite window of test depths with exact recovery. A single depth
chosen from the training distribution can solve harder test graphs, with
a guaranteed extrapolation gain that grows exponentially with the gap
between training depth and graph hardness. Experiments validate these
windows and show how compute and initialization affect the choice of
training depth.

Our analysis assumes linear attention, row normalization, population
training, and symmetric initialization. Extending the theory to softmax
attention and finite-sample training is an important direction. The
sufficient test-depth window could also be tightened, as our ER
experiments show successful depths beyond those certified by the theory.
Whether similar behavior arises on other algorithmic tasks remains open.

\section*{Acknowledgements}
This research was supported in part by the NSF CAREER Award 2146492, NSF-Simons AI Institute for Cosmic Origins (CosmicAI), and NSF AI Institute for Foundations of Machine Learning (IFML).
 %

\nocite{wei2022chain,snell2025scaling}
\bibliographystyle{plainnat}
\begingroup
\small
\setlength{\bibsep}{4pt plus 1pt}
\interlinepenalty=10000
\bibliography{references}

@inproceedings{
lan2020albert,
title={ALBERT: A Lite BERT for Self-supervised Learning of Language Representations},
author={Zhenzhong Lan and Mingda Chen and Sebastian Goodman and Kevin Gimpel and Piyush Sharma and Radu Soricut},
booktitle={International Conference on Learning Representations},
year={2020},
url={https://openreview.net/forum?id=H1eA7AEtvS}
}

@article{perez2021attention,
  author  = {Jorge Pérez and Pablo Barceló and Javier Marinkovic},
  title   = {Attention is Turing-Complete},
  journal = {Journal of Machine Learning Research},
  year    = {2021},
  volume  = {22},
  number  = {75},
  pages   = {1--35},
  url     = {http://jmlr.org/papers/v22/20-302.html}
}

@InProceedings{sanford2024transformers,
  title = 	 {Transformers, parallel computation, and logarithmic depth},
  author =       {Sanford, Clayton and Hsu, Daniel and Telgarsky, Matus},
  booktitle = 	 {Proceedings of the 41st International Conference on Machine Learning},
  pages = 	 {43276--43327},
  year = 	 {2024},
  editor = 	 {Salakhutdinov, Ruslan and Kolter, Zico and Heller, Katherine and Weller, Adrian and Oliver, Nuria and Scarlett, Jonathan and Berkenkamp, Felix},
  volume = 	 {235},
  series = 	 {Proceedings of Machine Learning Research},
  month = 	 {21--27 Jul},
  publisher =    {PMLR},
  url = 	 {https://proceedings.mlr.press/v235/sanford24a.html}
}

@inproceedings{garg2022what,
 author = {Garg, Shivam and Tsipras, Dimitris and Liang, Percy S and Valiant, Gregory},
 booktitle = {Advances in Neural Information Processing Systems},
 doi = {10.52202/068431-2217},
 editor = {S. Koyejo and S. Mohamed and A. Agarwal and D. Belgrave and K. Cho and A. Oh},
 pages = {30583--30598},
 publisher = {Curran Associates, Inc.},
 title = {What Can Transformers Learn In-Context? A Case Study of Simple Function Classes},
 url = {https://proceedings.neurips.cc/paper_files/paper/2022/file/c529dba08a146ea8d6cf715ae8930cbe-Paper-Conference.pdf},
 volume = {35},
 year = {2022}
}

@InProceedings{vonoswald2023transformers,
  title = 	 {Transformers Learn In-Context by Gradient Descent},
  author =       {Von Oswald, Johannes and Niklasson, Eyvind and Randazzo, Ettore and Sacramento, Joao and Mordvintsev, Alexander and Zhmoginov, Andrey and Vladymyrov, Max},
  booktitle = 	 {Proceedings of the 40th International Conference on Machine Learning},
  pages = 	 {35151--35174},
  year = 	 {2023},
  editor = 	 {Krause, Andreas and Brunskill, Emma and Cho, Kyunghyun and Engelhardt, Barbara and Sabato, Sivan and Scarlett, Jonathan},
  volume = 	 {202},
  series = 	 {Proceedings of Machine Learning Research},
  month = 	 {23--29 Jul},
  publisher =    {PMLR},
  url = 	 {https://proceedings.mlr.press/v202/von-oswald23a.html}
}

@inproceedings{ahn2023transformers,
 author = {Ahn, Kwangjun and Cheng, Xiang and Daneshmand, Hadi and Sra, Suvrit},
 booktitle = {Advances in Neural Information Processing Systems},
 doi = {10.52202/075280-1977},
 editor = {A. Oh and T. Naumann and A. Globerson and K. Saenko and M. Hardt and S. Levine},
 pages = {45614--45650},
 publisher = {Curran Associates, Inc.},
 title = {Transformers learn to implement preconditioned gradient descent for in-context learning},
 url = {https://proceedings.neurips.cc/paper_files/paper/2023/file/8ed3d610ea4b68e7afb30ea7d01422c6-Paper-Conference.pdf},
 volume = {36},
 year = {2023}
}

@article{zhang2024trained,
  author  = {Ruiqi Zhang and Spencer Frei and Peter L. Bartlett},
  title   = {Trained Transformers Learn Linear Models In-Context},
  journal = {Journal of Machine Learning Research},
  year    = {2024},
  volume  = {25},
  number  = {49},
  pages   = {1--55},
  url     = {http://jmlr.org/papers/v25/23-1042.html}
}

@InProceedings{chen2025bypassing,
  title = 	 {Bypassing the Exponential Dependency: Looped Transformers Efficiently Learn In-context by Multi-step Gradient Descent},
  author =       {Chen, Bo and Li, Xiaoyu and Liang, Yingyu and Shi, Zhenmei and Song, Zhao},
  booktitle = 	 {Proceedings of The 28th International Conference on Artificial Intelligence and Statistics},
  pages = 	 {4447--4455},
  year = 	 {2025},
  editor = 	 {Li, Yingzhen and Mandt, Stephan and Agrawal, Shipra and Khan, Emtiyaz},
  volume = 	 {258},
  series = 	 {Proceedings of Machine Learning Research},
  month = 	 {03--05 May},
  publisher =    {PMLR},
  url = 	 {https://proceedings.mlr.press/v258/chen25i.html}
}

@inproceedings{
xu2021neural,
title={How Neural Networks Extrapolate: From Feedforward to Graph Neural Networks},
author={Keyulu Xu and Mozhi Zhang and Jingling Li and Simon Shaolei Du and Ken-Ichi Kawarabayashi and Stefanie Jegelka},
booktitle={International Conference on Learning Representations},
year={2021},
url={https://openreview.net/forum?id=UH-cmocLJC}
}

@InProceedings{yehudai2021local,
  title = 	 {From Local Structures to Size Generalization in Graph Neural Networks},
  author =       {Yehudai, Gilad and Fetaya, Ethan and Meirom, Eli and Chechik, Gal and Maron, Haggai},
  booktitle = 	 {Proceedings of the 38th International Conference on Machine Learning},
  pages = 	 {11975--11986},
  year = 	 {2021},
  editor = 	 {Meila, Marina and Zhang, Tong},
  volume = 	 {139},
  series = 	 {Proceedings of Machine Learning Research},
  month = 	 {18--24 Jul},
  publisher =    {PMLR},
  url = 	 {https://proceedings.mlr.press/v139/yehudai21a.html}
}

@inproceedings{
velickovic2020neural,
title={Neural Execution of Graph Algorithms},
author={Petar Veličković and Rex Ying and Matilde Padovano and Raia Hadsell and Charles Blundell},
booktitle={International Conference on Learning Representations},
year={2020},
url={https://openreview.net/forum?id=SkgKO0EtvS}
}

@InProceedings{velickovic2022clrs,
  title = 	 {The {CLRS} Algorithmic Reasoning Benchmark},
  author =       {Veli{\v{c}}kovi{\'c}, Petar and Badia, Adri{\`a} Puigdom{\`e}nech and Budden, David and Pascanu, Razvan and Banino, Andrea and Dashevskiy, Misha and Hadsell, Raia and Blundell, Charles},
  booktitle = 	 {Proceedings of the 39th International Conference on Machine Learning},
  pages = 	 {22084--22102},
  year = 	 {2022},
  editor = 	 {Chaudhuri, Kamalika and Jegelka, Stefanie and Song, Le and Szepesvari, Csaba and Niu, Gang and Sabato, Sivan},
  volume = 	 {162},
  series = 	 {Proceedings of Machine Learning Research},
  month = 	 {17--23 Jul},
  publisher =    {PMLR},
  url = 	 {https://proceedings.mlr.press/v162/velickovic22a.html}
}

@inproceedings{
dehghani2019universal,
title={Universal Transformers},
author={Mostafa Dehghani and Stephan Gouws and Oriol Vinyals and Jakob Uszkoreit and Lukasz Kaiser},
booktitle={International Conference on Learning Representations},
year={2019},
url={https://openreview.net/forum?id=HyzdRiR9Y7},
}

@InProceedings{giannou2023looped,
  title = 	 {Looped Transformers as Programmable Computers},
  author =       {Giannou, Angeliki and Rajput, Shashank and Sohn, Jy-Yong and Lee, Kangwook and Lee, Jason D. and Papailiopoulos, Dimitris},
  booktitle = 	 {Proceedings of the 40th International Conference on Machine Learning},
  pages = 	 {11398--11442},
  year = 	 {2023},
  editor = 	 {Krause, Andreas and Brunskill, Emma and Cho, Kyunghyun and Engelhardt, Barbara and Sabato, Sivan and Scarlett, Jonathan},
  volume = 	 {202},
  series = 	 {Proceedings of Machine Learning Research},
  month = 	 {23--29 Jul},
  publisher =    {PMLR},
  url = 	 {https://proceedings.mlr.press/v202/giannou23a.html}
}

@inproceedings{yang2024looped,
 author = {Yang, Liu and Lee, Kangwook and Nowak, Robert and Papailiopoulos, Dimitris},
 booktitle = {International Conference on Learning Representations},
 editor = {B. Kim and Y. Yue and S. Chaudhuri and K. Fragkiadaki and M. Khan and Y. Sun},
 pages = {42195--42214},
 title = {Looped Transformers are Better at Learning Learning Algorithms},
 url = {https://proceedings.iclr.cc/paper_files/paper/2024/file/b8402301e7f06bdc97a31bfaa653dc32-Paper-Conference.pdf},
 volume = {2024},
 year = {2024}
}

@inproceedings{geiping2025scaling,
 author = {Geiping, Jonas and McLeish, Sean and Jain, Neel and Kirchenbauer, John and Singh, Siddharth and Bartoldson, Brian and Kailkhura, Bhavya and Bhatele, Abhinav and Goldstein, Tom},
 booktitle = {Advances in Neural Information Processing Systems},
 doi = {10.52202/085713-1380},
 editor = {D. Belgrave and C. Zhang and H. Lin and R. Pascanu and P. Koniusz and M. Ghassemi and N. Chen},
 pages = {41340--41391},
 publisher = {Curran Associates, Inc.},
 title = {Scaling up Test-Time Compute with Latent Reasoning: A Recurrent Depth Approach},
 url = {https://proceedings.neurips.cc/paper_files/paper/2025/file/3b01972cf31e6fa0fe29e4b8b5c2a0a1-Paper-Conference.pdf},
 volume = {38, Main Conference},
 year = {2025}
}

@inproceedings{saunshi2025reasoning,
 author = {Saunshi, Nikunj and Dikkala, Nishanth and Li, Zhiyuan and Kumar, Sanjiv and J. Reddi, Sashank},
 booktitle = {International Conference on Learning Representations},
 editor = {Y. Yue and A. Garg and N. Peng and F. Sha and R. Yu},
 pages = {14855--14881},
 title = {Reasoning with Latent Thoughts: On the Power of Looped Transformers},
 url = {https://proceedings.iclr.cc/paper_files/paper/2025/file/2676109d49d1eb26d6bc584a8f556305-Paper-Conference.pdf},
 volume = {2025},
 year = {2025}
}

@inproceedings{wei2022chain,
 author = {Wei, Jason and Wang, Xuezhi and Schuurmans, Dale and Bosma, Maarten and ichter, brian and Xia, Fei and Chi, Ed and Le, Quoc V and Zhou, Denny},
 booktitle = {Advances in Neural Information Processing Systems},
 doi = {10.52202/068431-1800},
 editor = {S. Koyejo and S. Mohamed and A. Agarwal and D. Belgrave and K. Cho and A. Oh},
 pages = {24824--24837},
 publisher = {Curran Associates, Inc.},
 title = {Chain-of-Thought Prompting Elicits Reasoning in Large Language Models},
 url = {https://proceedings.neurips.cc/paper_files/paper/2022/file/9d5609613524ecf4f15af0f7b31abca4-Paper-Conference.pdf},
 volume = {35},
 year = {2022}
}

@inproceedings{snell2025scaling,
 author = {Snell, Charlie and Lee, Jaehoon and Xu, Kelvin and Kumar, Aviral},
 booktitle = {International Conference on Learning Representations},
 editor = {Y. Yue and A. Garg and N. Peng and F. Sha and R. Yu},
 pages = {10131--10165},
 title = {Scaling LLM Test-Time Compute Optimally Can be More Effective than Scaling Parameters for Reasoning},
 url = {https://proceedings.iclr.cc/paper_files/paper/2025/file/1b623663fd9b874366f3ce019fdfdd44-Paper-Conference.pdf},
 volume = {2025},
 year = {2025}
}

@inproceedings{schwarzschild2021can,
 author = {Schwarzschild, Avi and Borgnia, Eitan and Gupta, Arjun and Huang, Furong and Vishkin, Uzi and Goldblum, Micah and Goldstein, Tom},
 booktitle = {Advances in Neural Information Processing Systems},
 editor = {M. Ranzato and A. Beygelzimer and Y. Dauphin and P.S. Liang and J. Wortman Vaughan},
 pages = {6695--6706},
 publisher = {Curran Associates, Inc.},
 title = {Can You Learn an Algorithm?  Generalizing from Easy to Hard Problems with Recurrent Networks},
 url = {https://proceedings.neurips.cc/paper_files/paper/2021/file/3501672ebc68a5524629080e3ef60aef-Paper.pdf},
 volume = {34},
 year = {2021}
}

@inproceedings{fan2025looped,
 author = {Fan, Ying and Du, Yilun and Ramchandran, Kannan and Lee, Kangwook},
 booktitle = {International Conference on Learning Representations},
 editor = {Y. Yue and A. Garg and N. Peng and F. Sha and R. Yu},
 pages = {14502--14520},
 title = {Looped Transformers for Length Generalization},
 url = {https://proceedings.iclr.cc/paper_files/paper/2025/file/25cc3adf8c85f7c70989cb8a97a691a7-Paper-Conference.pdf},
 volume = {2025},
 year = {2025}
}

@inproceedings{bansal2022end,
 author = {Bansal, Arpit and Schwarzschild, Avi and Borgnia, Eitan and Emam, Zeyad and Huang, Furong and Goldblum, Micah and Goldstein, Tom},
 booktitle = {Advances in Neural Information Processing Systems},
 doi = {10.52202/068431-1471},
 editor = {S. Koyejo and S. Mohamed and A. Agarwal and D. Belgrave and K. Cho and A. Oh},
 pages = {20232--20242},
 publisher = {Curran Associates, Inc.},
 title = {End-to-end Algorithm Synthesis with Recurrent Networks: Extrapolation without Overthinking},
 url = {https://proceedings.neurips.cc/paper_files/paper/2022/file/7f70331dbe58ad59d83941dfa7d975aa-Paper-Conference.pdf},
 volume = {35},
 year = {2022}
}

@inproceedings{anil2022exploring,
 author = {Anil, Cem and Wu, Yuhuai and Andreassen, Anders and Lewkowycz, Aitor and Misra, Vedant and Ramasesh, Vinay and Slone, Ambrose and Gur-Ari, Guy and Dyer, Ethan and Neyshabur, Behnam},
 booktitle = {Advances in Neural Information Processing Systems},
 doi = {10.52202/068431-2793},
 editor = {S. Koyejo and S. Mohamed and A. Agarwal and D. Belgrave and K. Cho and A. Oh},
 pages = {38546--38556},
 publisher = {Curran Associates, Inc.},
 title = {Exploring Length Generalization in Large Language Models},
 url = {https://proceedings.neurips.cc/paper_files/paper/2022/file/fb7451e43f9c1c35b774bcfad7a5714b-Paper-Conference.pdf},
 volume = {35},
 year = {2022}
}

@inproceedings{zhou2024algorithms,
 author = {Zhou, Hattie and Bradley, Arwen and Littwin, Etai and Razin, Noam and Saremi, Omid and Susskind, Joshua and Bengio, Samy and Nakkiran, Preetum},
 booktitle = {International Conference on Learning Representations},
 editor = {B. Kim and Y. Yue and S. Chaudhuri and K. Fragkiadaki and M. Khan and Y. Sun},
 pages = {15898--15926},
 title = {What Algorithms can Transformers Learn? A Study in Length Generalization},
 url = {https://proceedings.iclr.cc/paper_files/paper/2024/file/45ed1a72597594c097152ef9cc187762-Paper-Conference.pdf},
 volume = {2024},
 year = {2024}
}

@InProceedings{deluca2024simulation,
  title = 	 {Simulation of Graph Algorithms with Looped Transformers},
  author =       {Back De Luca, Artur and Fountoulakis, Kimon},
  booktitle = 	 {Proceedings of the 41st International Conference on Machine Learning},
  pages = 	 {2319--2363},
  year = 	 {2024},
  editor = 	 {Salakhutdinov, Ruslan and Kolter, Zico and Heller, Katherine and Weller, Adrian and Oliver, Nuria and Scarlett, Jonathan and Berkenkamp, Felix},
  volume = 	 {235},
  series = 	 {Proceedings of Machine Learning Research},
  month = 	 {21--27 Jul},
  publisher =    {PMLR},
  url = 	 {https://proceedings.mlr.press/v235/back-de-luca24a.html}
}

@inproceedings{sanford2024understanding,
 author = {Sanford, Clayton and Fatemi, Bahare and Hall, Ethan and Tsitsulin, Anton and Kazemi, Mehran and Halcrow, Jonathan and Perozzi, Bryan and Mirrokni, Vahab},
 booktitle = {Advances in Neural Information Processing Systems},
 doi = {10.52202/079017-2490},
 editor = {A. Globerson and L. Mackey and D. Belgrave and A. Fan and U. Paquet and J. Tomczak and C. Zhang},
 pages = {78320--78370},
 publisher = {Curran Associates, Inc.},
 title = {Understanding Transformer Reasoning Capabilities via Graph Algorithms},
 url = {https://proceedings.neurips.cc/paper_files/paper/2024/file/8f395480c04ac6dfb2c2326a639df88e-Paper-Conference.pdf},
 volume = {37},
 year = {2024}
}

@inproceedings{merrill2025little,
 author = {Merrill, Will and Sabharwal, Ashish},
 booktitle = {Advances in Neural Information Processing Systems},
 doi = {10.52202/085713-3186},
 editor = {D. Belgrave and C. Zhang and H. Lin and R. Pascanu and P. Koniusz and M. Ghassemi and N. Chen},
 pages = {95315--95339},
 publisher = {Curran Associates, Inc.},
 title = {A Little Depth Goes a Long Way: The Expressive Power of Log-Depth Transformers},
 url = {https://proceedings.neurips.cc/paper_files/paper/2025/file/88dd7aa6979e352fda7c4952ca8eac59-Paper-Conference.pdf},
 volume = {38, Main Conference},
 year = {2025}
}

@InProceedings{gatmiry2024looped,
  title = 	 {Can Looped Transformers Learn to Implement Multi-step Gradient Descent for In-context Learning?},
  author =       {Gatmiry, Khashayar and Saunshi, Nikunj and J. Reddi, Sashank and Jegelka, Stefanie and Kumar, Sanjiv},
  booktitle = 	 {Proceedings of the 41st International Conference on Machine Learning},
  pages = 	 {15130--15152},
  year = 	 {2024},
  editor = 	 {Salakhutdinov, Ruslan and Kolter, Zico and Heller, Katherine and Weller, Adrian and Oliver, Nuria and Scarlett, Jonathan and Berkenkamp, Felix},
  volume = 	 {235},
  series = 	 {Proceedings of Machine Learning Research},
  month = 	 {21--27 Jul},
  publisher =    {PMLR},
  url = 	 {https://proceedings.mlr.press/v235/gatmiry24b.html}
}

@misc{wu2026power,
      title={Looped Transformers with Layer Normalization Provably Learn the Power Method}, 
      author={Lyumin Wu and Chenyang Zhang and Yuan Cao},
      year={2026},
      eprint={2606.00605},
      archivePrefix={arXiv},
      primaryClass={cs.LG},
      url={https://arxiv.org/abs/2606.00605}, 
}

@InProceedings{ye2026connectivity,
  title = 	 {Transformers Provably Learn Algorithmic Solutions for Graph Connectivity, But Only with the Right Data},
  author =       {Ye, Qilin and Fu, Deqing and Jia, Robin and Sharan, Vatsal},
  booktitle = 	 {Proceedings of the 43rd International Conference on Machine Learning},
  pages = 	 {148293--148324},
  year = 	 {2026},
  editor = 	 {Zhang, Tong and Dudik, Miroslav and Jaggi, Martin and Agarwal, Alekh and Li, Sharon and Schuurmans, Dale and Zhu, Jerry and Berkenkamp, Felix and Dong, Hanze and Bietti, Alberto},
  volume = 	 {306},
  series = 	 {Proceedings of Machine Learning Research},
  month = 	 {06--11 Jul},
  publisher =    {PMLR},
  url = 	 {https://proceedings.mlr.press/v306/ye26f.html}
}
\endgroup

\clearpage
\appendix
\raggedbottom
\setcounter{topnumber}{3}
\setcounter{totalnumber}{4}
\renewcommand{\topfraction}{0.95}
\renewcommand{\textfraction}{0.04}

\section{Problem Setting and Training: Supporting Results}
\label{app:setting-training}
Throughout the appendices, the threshold $\tau\in(0,1)$ is
fixed and is shared by training and testing. We do not study its
variation. Constants in $O(\cdot)$, $\Theta(\cdot)$, and inequalities
may depend on this fixed threshold, as well as on other explicitly
stated fixed parameters.
All experiments and numerical illustrations use $\tau=1/2$.

For the exact certificates, it is convenient to set once and for all
\[
c_{\rm in}:=\log\frac{1+\tau}{1-\tau}>0,
\qquad
c_{\rm out}:=\log\frac1{1-\tau}>0.
\]

\subsection{Graph spectrum and connectivity}
\label{app:proof:prop:spectrum}

For an undirected, unweighted graph $G$, let
$d_G=\max\{1,d_{\max}(G)\}$. Its representation is
$P_G=I-L_G/(2d_G)$. For components $C_1,\ldots,C_{c_G}$, the target is
\begin{equation}
\label{eq:target}
\Pi_G=\sum_{j=1}^{c_G}
\frac{\mathbf 1_{C_j}\mathbf 1_{C_j}^{\top}}{|C_j|}.
\end{equation}
Here $\mathbf 1_C$ is the indicator vector of $C$.

\begin{proposition}[Spectral structure]
\label{prop:spectrum}
The matrix $P_G$ is symmetric, positive semidefinite, and row
stochastic. Its eigenvalue-one eigenspace consists of the vectors
constant on each component, and its orthogonal projector is $\Pi_G$.
For an orthonormal eigenbasis,
\begin{equation}
\label{eq:spectral-decomp}
P_G^m=\Pi_G+\sum_{i>c_G}\lambda_i^m u_i u_i^\top,
\qquad m\ge1,
\end{equation}
where $0\le\lambda_i<1$. In particular, $P_G^m\to\Pi_G$.
\end{proposition}

\begin{proof}
The off-diagonal entries of $P_G$ are nonnegative, its diagonal
entries are at least $1/2$, and $L_G\mathbf 1=0$ gives
$P_G\mathbf 1=\mathbf 1$. For every $v\in\mathbb R^N$,
\[
\begin{aligned}
v^\top L_Gv
&=\sum_{\{i,j\}\in\mathcal E(G)}(v_i-v_j)^2\\
&\le2d_{\max}(G)\|v\|_2^2.
\end{aligned}
\]
Thus the eigenvalues of $L_G$ lie in $[0,2d_G]$, and those of $P_G$
lie in $[0,1]$. The equality $v^\top L_Gv=0$ holds exactly when
$v_i=v_j$ on every edge, or equivalently when $v$ is constant on
each component. The normalized component indicators form an
orthonormal basis of this space, giving~\eqref{eq:target}.
The spectral theorem then yields~\eqref{eq:spectral-decomp} and the
limit, since all remaining eigenvalues are strictly below one.
\end{proof}

\subsection{Vertex relabeling and parameter invariance}
\label{app:proof:prop:invariance}

\begin{assumption}[Exchangeable training distribution]
\label{assump:exchangeable}
Let $\mathcal D_{\mathrm{tr}}$ be a distribution over graphs on $[N]$.
For every permutation $\pi$ of $[N]$, relabeling vertex $i$ as
$\pi(i)$ preserves the distribution:
$G\sim\mathcal D_{\mathrm{tr}}$ implies
$\pi(G)\sim\mathcal D_{\mathrm{tr}}$.
\end{assumption}

Let $Q_\pi$ satisfy $Q_\pi\mathbf e_i=\mathbf e_{\pi(i)}$.
Then $P_{\pi(G)}=Q_\pi P_GQ_\pi^\top$ and
$\Pi_{\pi(G)}=Q_\pi\Pi_GQ_\pi^\top$.
Set $E=\mathbf 1\mathbf 1^\top/N$ and $F=I-E$.

\begin{proposition}[$E$--$F$ invariance]
\label{prop:invariance}
Under Assumption~\ref{assump:exchangeable}, population matrix gradient
flow and gradient descent preserve the forms $W=aF+bE$ and
$V=cF+dE$, wherever the loss and updates are defined.
\end{proposition}

\begin{proof}
Write $Q=Q_\pi$. Row normalization commutes with conjugation:
\[
\mathcal N(QZQ^\top)=Q\mathcal N(Z)Q^\top.
\]
Indeed, $Q^\top\mathbf 1=\mathbf 1$ and
$\operatorname{diag}(Qv)=Q\operatorname{diag}(v)Q^\top$.
The block therefore commutes with simultaneous conjugation of its
input and weights. By induction, its output obeys the same relation.
Exchangeability and Frobenius-norm invariance give
\[
\mathcal L_K(QWQ^\top,QVQ^\top)=\mathcal L_K(W,V).
\]
Writing $W'=QWQ^\top$ and $V'=QVQ^\top$, differentiation yields
\[
\nabla_W\mathcal L_K(W',V')
=Q\nabla_W\mathcal L_K(W,V)Q^\top,
\]
and the same identity for $V$.

If $W,V\in\operatorname{span}\{E,F\}$, both weights are fixed by
every permutation conjugation. Their gradients are therefore fixed
as well. A matrix fixed by every such conjugation has one common
diagonal entry and one common off-diagonal entry; it belongs to
$\operatorname{span}\{E,F\}$. Thus the vector field and every defined
discrete update preserve this subspace. Positivity and global
well-definedness under the stated training conditions are established
below.
\end{proof}

\subsection{Closed-form loop recursion}
\label{app:proof:prop:forward}

\begin{proposition}[Exact forward map]
\label{prop:forward}
For $W=aF+bE$, $V=cF+dE$, and $a,b,c,d>0$, define
\begin{equation}
\label{eq:effective}
\gamma=\frac{ac}{bd},\qquad e_t=\frac{3^t-1}{2}.
\end{equation}
Then $H_t=E+\gamma^{e_t}(P_G^{3^t}-E)$ for every $t\ge0$.
\end{proposition}

\begin{proof}
Since $P_GE=EP_G=E$, all matrices appearing in the claimed formula
commute with $W$ and $V$. The formula is true at $t=0$.
If it holds at $t$, orthogonality of the uniform and nonuniform
subspaces gives
\[
H_t^3=E+\gamma^{3e_t}(P_G^{3^{t+1}}-E).
\]
The block output before normalization equals $H_t^3WV$, since $H_t$
is symmetric and the matrices commute. As $WV=acF+bdE$,
\[
H_t^3WV=bdE+ac\gamma^{3e_t}(P_G^{3^{t+1}}-E).
\]
Its row sum is $bd>0$. Normalization therefore gives
$H_{t+1}=E+\gamma^{3e_t+1}(P_G^{3^{t+1}}-E)$.
The identity $e_{t+1}=3e_t+1$ completes the induction.
\end{proof}

\subsection{Effective loss and stationary points}
\label{app:proof:prop:effective-loss}

The training statistics are
\begin{equation}
\label{eq:statistics}
\begin{aligned}
\kappa&=\mathbb E_G(c_G-1)>0,\\
\delta_K&=\mathbb E_G\|P_G^{3^K}-\Pi_G\|_F^2\\
&=\mathbb E_G\sum_{i>c_G}\lambda_i^{2\cdot3^K}.
\end{aligned}
\end{equation}
Expectations are taken under the fixed population training law.

\begin{proposition}[One-dimensional objective]
\label{prop:effective-loss}
On the positive invariant parameter family,
\begin{equation}
\label{eq:effective-loss}
\mathcal L_K(\gamma)
=\frac12\bigl[\kappa(1-\gamma^{e_K})^2
+\delta_K\gamma^{2e_K}\bigr].
\end{equation}
Its unique minimizing ratio is $\gamma_*$ in~\eqref{eq:opt}. The
minimum $\mathcal L_*=\kappa\delta_K/[2(\kappa+\delta_K)]$ satisfies
\begin{equation}
\label{eq:loss}
\mathcal L_K-\mathcal L_*
=\frac{\kappa+\delta_K}{2}
\bigl(\gamma^{e_K}-\gamma_*^{e_K}\bigr)^2.
\end{equation}
Every positive restricted stationary pair has this ratio and is also
stationary for the full matrix objective under exchangeability Assumption~\ref{assump:exchangeable}.
\end{proposition}

\begin{proof}
The exact output gives
\[
\begin{aligned}
H_K-\Pi_G={}&\gamma^{e_K}(P_G^{3^K}-\Pi_G)-(1-\gamma^{e_K})(\Pi_G-E).
\end{aligned}
\]
The two terms are orthogonal in Frobenius inner product, since they
act on orthogonal eigenspaces. Also $\Pi_G-E$ is an orthogonal
projector of rank $c_G-1$. Squaring and taking expectations yields
\eqref{eq:effective-loss}. As a function of $q=\gamma^{e_K}>0$, this
is a strictly convex quadratic with minimizer
$q_*=\kappa/(\kappa+\delta_K)$. Completing the square gives
\eqref{eq:loss} and the minimum.

The gradient of $\gamma$ does not vanish on the positive orthant.
Hence restricted stationarity is equivalent to $\gamma=\gamma_*$.
Proposition~\ref{prop:invariance} places the full gradient in the
invariant subspace, so a zero restricted gradient is a zero full
gradient. This establishes the minimum within the invariant family.
\end{proof}

\subsection{Population gradient-flow convergence}
\label{app:proof:thm:gf}

The coordinates
\begin{equation}
\label{eq:rescaled-coordinate}
x=(\sqrt{N-1}\,a,b,\sqrt{N-1}\,c,d)
\end{equation}
identify the invariant parameter subspace isometrically with Euclidean
space: $\|W\|_F^2+\|V\|_F^2=\|x\|_2^2$.
Thus full matrix gradient flow induces ordinary Euclidean gradient
flow in $x$. In these coordinates,
$\gamma=x_1x_3/[(N-1)x_2x_4]$.

\begin{lemma}[Gradient inequality]
\label{lem:pl}
On a region where $x_j>0$, $\gamma\ge\underline\gamma>0$, and
$\|x\|_2\le M$, one has
\[
\|\nabla_x\mathcal L_K\|_2^2
\ge2\mu(\mathcal L_K-\mathcal L_*),
\]
with
\begin{equation}
\label{eq:pl-constant}
\mu=\frac{16(\kappa+\delta_K)e_K^2
\underline\gamma^{2e_K}}{M^2}.
\end{equation}
\end{lemma}

\begin{proof}
Differentiating~\eqref{eq:loss} through $\gamma$ yields
\[
\begin{aligned}
\|\nabla_x\mathcal L_K\|_2^2
={}&2(\kappa+\delta_K)e_K^2\gamma^{2e_K}\cdot(\mathcal L_K-\mathcal L_*)
\sum_{j=1}^4x_j^{-2}.
\end{aligned}
\]
Cauchy--Schwarz gives $\sum_jx_j^{-2}\ge16/\|x\|_2^2$, proving
the claim.
\end{proof}

\begin{proof}[Proof of Theorem~\ref{thm:gf}]
Work first on the maximal positive trajectory. The chain rule gives
\[
\dot\gamma=-\mathcal L_K'(\gamma)\|\nabla_x\gamma\|_2^2.
\]
The derivative $\mathcal L_K'(\gamma)$ has the sign of
$\gamma-\gamma_*$. Hence $\gamma$ stays between
$\underline\gamma=\min\{\gamma_0,\gamma_*\}$ and
$\overline\gamma=\max\{\gamma_0,\gamma_*\}$.
For $i\in\{1,3\}$ and $j\in\{2,4\}$,
\[
\dot x_i=-\mathcal L_K'(\gamma)\frac{\gamma}{x_i},
\qquad
\dot x_j=\mathcal L_K'(\gamma)\frac{\gamma}{x_j}.
\]
Consequently $\|x\|_2^2$ and every opposite-sign pair sum
$x_i^2+x_j^2$ are conserved.

These identities keep every coordinate uniformly away from zero.
Indeed, if a numerator coordinate approached zero along a sequence,
each denominator coordinate would have a positive lower bound from
its conserved pair sum; boundedness of $x$ would then force
$\gamma\to0$. If a denominator coordinate approached zero, the
numerator coordinates would instead stay positive, forcing
$\gamma\to\infty$. Both contradict the ratio bounds.
Thus the trajectory stays in a compact subset of the positive
orthant and exists for all optimization times $r\ge0$.

Apply Lemma~\ref{lem:pl} with $M=\|x_0\|_2$ to obtain
\[
\frac{d}{dr}(\mathcal L_K-\mathcal L_*)
\le-2\mu(\mathcal L_K-\mathcal L_*).
\]
Integration proves the loss bound in~\eqref{eq:gf-rates}.
Equation~\eqref{eq:loss} then gives
\[
|\gamma(r)^{e_K}-\gamma_*^{e_K}|
\le|\gamma_0^{e_K}-\gamma_*^{e_K}|e^{-\mu r}.
\]
Since $e_K\ge1$, the mean value theorem yields the ratio bound with
\begin{equation}
\label{eq:gamma-constant}
C=\frac{|\gamma_0^{e_K}-\gamma_*^{e_K}|}
{e_K\underline\gamma^{e_K-1}}.
\end{equation}
Finally, $\nabla_x(\gamma^{e_K})$ is bounded on the compact
trajectory, so the gradient norm is at most a constant times
$e^{-\mu r}$. The trajectory has finite length and converges to a
positive pair with ratio $\gamma_*$. Proposition~\ref{prop:effective-loss}
identifies this pair as a stationary point of the full objective.
\end{proof}

\paragraph{Depth dependence of the rate.}
In the Frobenius coordinates~\eqref{eq:rescaled-coordinate},
with $e_K=(3^K-1)/2$, the proof permits
\begin{equation}
\label{eq:mu-depth}
\mu=\frac{16(\kappa+\delta_K)e_K^2
\min\{\gamma_0,\gamma_*\}^{2e_K}}{\|x_0\|_2^2}.
\end{equation}
When a fixed $\gamma_0<1$ lies below $\gamma_*$, the factor
$\gamma_0^{3^K-1}$ can outweigh the growing $e_K^2$, giving a
small rate certificate at large $K$.

\section{Discrete Population Gradient Descent}
\label{app:proof:thm:gd}

The Euclidean coordinates~\eqref{eq:rescaled-coordinate} also describe
full matrix gradient descent. Write
$\underline\gamma=\min\{\gamma_0,\gamma_*\}$ and
$\overline\gamma=\max\{\gamma_0,\gamma_*\}$, and in this section use
\begin{equation}
\label{eq:gd-norm-bound}
M^2=\|x_0\|_2^2+2(\mathcal L_K(x_0)-\mathcal L_*).
\end{equation}
Let $\mu$ be~\eqref{eq:pl-constant} with this $M$, and let $C$ be
\eqref{eq:gamma-constant}.

\begin{theorem}[Geometric convergence of gradient descent]
\label{thm:gd}
Under the assumptions of Theorem~\ref{thm:gf}, there is
$\eta_0>0$ such that every constant step size $0<\eta\le\eta_0$
keeps the iterates positive and gives
\begin{equation}
\label{eq:gd-rates}
\begin{aligned}
\mathcal L_K(x_n)-\mathcal L_*
&\le(1-\eta\mu)^n(\mathcal L_K(x_0)-\mathcal L_*),\\
|\gamma_n-\gamma_*|&\le C(1-\eta\mu)^{n/2}.
\end{aligned}
\end{equation}
The parameter matrices converge to a positive stationary pair with
effective ratio $\gamma_*$.
\end{theorem}

\subsection{An explicit sufficient step size}
\label{app:gd}

We construct an invariant region in which the ratio remains between its initial and optimal values, the parameter norm stays bounded, and every coordinate stays positive. The constants \(M,h,\ell\) control these three properties. A sufficiently small step then gives descent and prevents the ratio from crossing its optimum. Define:
\begin{equation}
\label{eq:gd-coordinate-bounds}
\begin{aligned}
h&=\min\{x_{0,1}^2,x_{0,3}^2\}
  +\min\{x_{0,2}^2,x_{0,4}^2\},\\
\ell&=\min\!\left\{\sqrt{h/2},\,
\frac{(N-1)\underline\gamma h}{2M},\,
\frac{h}{2(N-1)\overline\gamma M}\right\}.
\end{aligned}
\end{equation}
A sufficient choice is
\begin{equation}
\label{eq:gd-step-size}
\eta_0=\min\!\left\{1,\,
\frac{\ell^2}
{400(\kappa+\delta_K)e_K^2\overline\gamma^{2e_K}}
\right\}.
\end{equation}
The proof below establishes both this step-size bound and the
boundedness of the positive trajectory.

\begin{proof}[Proof of Theorem~\ref{thm:gd}]
For this proof only, abbreviate $e=e_K$, $q=\gamma^e$,
$q_*=\gamma_*^e$, and $A=\kappa+\delta_K$. Let
$\sigma=(1,-1,1,-1)$, so
$\partial_jq=eq\sigma_j/x_j$ and
$\nabla\mathcal L_K=A(q-q_*)\nabla q$.
We prove by induction that
\begin{equation}
\label{eq:gd-invariant-region}
x_{n,j}\ge\ell,\qquad
\|x_n\|_2\le M,\qquad
\underline\gamma\le\gamma_n\le\overline\gamma,
\end{equation}
and that each update satisfies descent.

\paragraph{The coordinate lower bound.}
Suppose a point has norm at most $M$, ratio between the two stated
bounds, and $x_i^2+x_j^2\ge h$ for every
$i\in\{1,3\}$, $j\in\{2,4\}$. If a numerator coordinate is
smaller than $\sqrt{h/2}$, both denominator coordinates exceed
$\sqrt{h/2}$. The product ratio then bounds that numerator below by
$(N-1)\underline\gamma h/(2M)$. Reversing the roles gives the
lower bound $h/[2(N-1)\overline\gamma M]$ for a small denominator
coordinate. Hence all coordinates are at least $\ell$.
The initial point satisfies these conditions, proving the base case.

\paragraph{1. Positive update segments.}
At an iterate satisfying~\eqref{eq:gd-invariant-region},
\[
\|\nabla\mathcal L_K(x_n)\|_2
\le\frac{2Ae\overline\gamma^{2e}}{\ell}.
\]
By~\eqref{eq:gd-step-size}, the step length is at most
$\ell/(200e)\le\ell/2$. Every coordinate on the segment from
$x_n$ to $x_{n+1}$ is therefore at least $\ell/2$.

\paragraph{2. Bounds along the segment.}
On this segment, $\|\nabla\log q\|_2\le4e/\ell$.
Consequently the change in $\log q$ from $x_n$ is at most
$8\eta Ae^2\overline\gamma^{2e}/\ell^2\le1$.
Thus $q\le3\overline\gamma^e$ throughout the segment. Also,
\[
\|\nabla\log q\|_2^2\le\frac{16e^2}{\ell^2},
\qquad
\|\nabla^2\log q\|_{\mathrm{op}}\le\frac{4e}{\ell^2}.
\]
Twice differentiating the loss gives
\[
\begin{aligned}
\nabla^2\mathcal L_K={}&A(2q-q_*)q
\nabla\log q\,\nabla\log q^\top\\
&+A(q-q_*)q\nabla^2\log q.
\end{aligned}
\]
Using $q_*\le\overline\gamma^e$ and $e\ge1$,
\[
\|\nabla^2\mathcal L_K\|_{\mathrm{op}}
\le\frac{384Ae^2\overline\gamma^{2e}}{\ell^2}.
\]
The chosen step size therefore ensures the descent inequality
\begin{equation}
\label{eq:gd-descent}
\mathcal L_K(x_{n+1})
\le\mathcal L_K(x_n)-\frac\eta2
\|\nabla\mathcal L_K(x_n)\|_2^2.
\end{equation}

\paragraph{3. No overshoot of the optimal ratio.}
For $x(v)=x_n-v\eta\nabla\mathcal L_K(x_n)$, the fundamental
theorem of calculus yields
\[
q_{n+1}-q_*=(1-\eta A m_n)(q_n-q_*),
\]
where
\[
m_n=\int_0^1
\langle\nabla q(x(v)),\nabla q(x_n)\rangle\,dv.
\]
The derivative signs are $\sigma$ throughout the positive orthant,
so $m_n>0$. The segment bounds imply
\[
m_n\le\frac{24e^2\overline\gamma^{2e}}{\ell^2}.
\]
Thus $0<1-\eta A m_n<1$, and $\gamma_{n+1}$ lies between
$\gamma_n$ and $\gamma_*$.

\paragraph{4. Norm and coordinate control.}
The ratio, and hence the loss, is invariant to common rescaling of
$x$. Therefore $\langle x,\nabla\mathcal L_K(x)\rangle=0$ and
\[
\|x_{n+1}\|_2^2-\|x_n\|_2^2
=\eta^2\|\nabla\mathcal L_K(x_n)\|_2^2.
\]
Summing with~\eqref{eq:gd-descent}, including the current step,
gives $\|x_{n+1}\|_2^2\le M^2$, since $\eta\le1$.
For $g=Aq(q-q_*)$, the coordinate update satisfies
\[
(x_j^+)^2=x_j^2-2\eta eg\sigma_j
+\eta^2e^2g^2/x_j^2.
\]
The linear terms cancel in an opposite-sign pair. Hence every
$x_i^2+x_j^2$ is nondecreasing and stays at least $h$.
The initial coordinate-bound argument now implies
$x_{n+1,j}\ge\ell$. This closes the induction.

\paragraph{5. Geometric convergence.}
Lemma~\ref{lem:pl} and~\eqref{eq:gd-descent} give
\[
\mathcal L_K(x_{n+1})-\mathcal L_*
\le(1-\eta\mu)(\mathcal L_K(x_n)-\mathcal L_*).
\]
Moreover $0<\eta\mu\le16/400<1$, by the definitions and
$\ell\le M$. Iteration proves the loss bound. Equation~\eqref{eq:loss}
and the mean value theorem prove the ratio bound with
$C$ in~\eqref{eq:gamma-constant}.
The gradient norm decays geometrically on the invariant compact
region. The sum of the step lengths is finite, so the weights converge
to a positive stationary pair with ratio $\gamma_*$.
\end{proof}

\section{Extrapolation: Proofs and Exact Conditions}
\label{app:extrapolation}

\subsection{Proof of the useful test-depth window}
\label{app:proof:prop:test-window}

\begin{proof}[Proof of Proposition~\ref{prop:test-window}]
Write $m=3^T$, $u=\gamma_*^{(m-1)/2}$, and $z=u^{-1}-1$.
Then $0<u\le1$, $z\ge0$, and the forward map gives
\[
H_T=u\left(P_G^m+\frac zN\mathbf 1\mathbf 1^\top\right).
\]
By Proposition~\ref{prop:spectrum}, the residual
$R=P_G^m-\Pi_G$ is positive semidefinite, with
$\|R\|_{\mathrm{op}}\le\rho^m$ and $|R_{ij}|\le\rho^m$.
In particular $(P_G^m)_{ii}\ge1/|C(i)|>0$, so the decoder is defined.
The common factor $u$ cancels in the normalized score.

\paragraph{Pairs in the same component.}
For $i,j$ in a component of size $s$, the score is bounded below by
\[
\frac{1/s+z/N-\rho^m}{1/s+z/N+\rho^m},
\]
provided its numerator is positive. It exceeds $\tau$ if
$(1+\tau)\rho^m<(1-\tau)/s$, which is guaranteed by
$\mathsf{Hard}(G)<T$.

\paragraph{Pairs in different components.}
Here $(P_G^m)_{ij}=0$. Using the diagonal lower bounds, the score
is at most
\[
\frac{z}{N/s_{\max}+z}<\tau
\quad\text{when}\quad
z<\frac{\tau N}{(1-\tau)s_{\max}}.
\]
Substituting
\[
z=\left(1+\frac{\delta_K}{\kappa}\right)^{
(3^T-1)/(3^K-1)}-1
\]
gives $T<T_{\max}(G;K)$ whenever $\delta_K>0$.
If $\delta_K=0$, then $z=0$ at every depth and no upper restriction
is needed. A connected graph has no cross-component pairs, so that
restriction may also be omitted in this case.
\end{proof}

\paragraph{Degenerate spectra.}
When $P_G=\Pi_G$, all non-unit eigenvalues are zero (or no such
eigenvalues exist). We set $\mathsf{Hard}(G)=-\infty$ in this case.
The within-component condition is then automatic, while the
cross-component restriction, when relevant, is unchanged.

\paragraph{Eventual false connections.}
For $\delta_K>0$, one has $\gamma_*<1$ and $H_T\to E$ as
$T\to\infty$. Every normalized score therefore tends to one.
More explicitly, $(P_G^m)_{ii}\le1$ bounds each cross-component
score below by $z/(N+z)$. Thus every such pair is predicted connected
once
\begin{equation}
\label{eq:late-failure}
\begin{aligned}
3^T>1
&+\frac{3^K-1}{\log(1+\delta_K/\kappa)}\\
&\!\cdot\log\!\left(1+\frac{\tau N}{1-\tau}\right).
\end{aligned}
\end{equation}
This sufficient failure condition is distinct from the upper
endpoint of the success window.

\subsection{Explicit certified region}
\label{app:solvable-region}

\begin{remark}[Integer test depths]
\label{rem:integer-window}
For finite $\mathsf{Hard}(G)$ and $\delta_K>0$, the displayed
window contains a nonnegative integer exactly when
\[
\lfloor\mathsf{Hard}(G)\rfloor+1
<
T_{\max}(G;K).
\]
For extrapolation strictly beyond the training depth, the condition is
\[
\max\!\left\{
K,\lfloor\mathsf{Hard}(G)\rfloor
\right\}+1
<
T_{\max}(G;K).
\]
Both endpoints in these tests belong to the same graph. The
explicit region in $(s_{\max},\rho)$ and degenerate cases follow.
\end{remark}

Assume $\delta_K>0$ and $0<\rho<1$ in this subsection. For brevity
only here, put $s=s_{\max}$ and define
\[
A(s)=\log s+c_{\rm in},
\qquad
B(s)=\log\!\left(1+\frac{\tau N}{(1-\tau)s}\right).
\]
The open continuous window is nonempty exactly when
\begin{equation}
\label{eq:continuous-region}
\frac{A(s)}{-\log\rho}
<1+\frac{(3^K-1)B(s)}{\log(1+\delta_K/\kappa)}.
\end{equation}
Equivalently, the boundary in the $(s,\rho)$ plane is
\begin{equation}
\label{eq:rho-region}
\rho<\exp\!\left[
-\frac{A(s)\log(1+\delta_K/\kappa)}
{\log(1+\delta_K/\kappa)+(3^K-1)B(s)}
\right].
\end{equation}
For fixed $K$, $A(s)$ increases and $B(s)$ decreases with $s$;
the boundary therefore decreases with component size.

\paragraph{Integer test loops.}
Write $h=\mathsf{Hard}(G)$ and $b=T_{\max}(G;K)>0$.
The smallest nonnegative integer strictly above $h$ is
$\max\{0,\lfloor h\rfloor+1\}$. Since $b>0$, an integer
$T\ge0$ with $h<T<b$ exists exactly when
\begin{equation}
\label{eq:integer-window}
\lfloor h\rfloor+1<b,
\quad\text{equivalently}\quad
h<\lceil b\rceil-1.
\end{equation}
The associated exact integer-window region is
\begin{equation}
\label{eq:integer-rho-region}
\rho<\exp\!\left[
-\frac{A(s)}{3^{\lceil T_{\max}(G;K)\rceil-1}}
\right].
\end{equation}
If $T>K$ is additionally required, the exact condition is
$\max\{K,\lfloor h\rfloor\}+1<b$.

For a fixed graph, $\lceil T_{\max}(G;K)\rceil-1$ is the
largest integer strictly below its upper endpoint. It is a
\emph{pointwise} quantity: changing the test component size changes
that endpoint. Equations~\eqref{eq:integer-window} and
\eqref{eq:integer-rho-region} characterize when the sufficient
window contains an integer depth. For $\mathsf{Hard}(G)=-\infty$, the lower restriction is
absent; for $\delta_K=0$, the upper restriction is absent. In
particular, the floor expression is not applied to $-\infty$.

\subsection{A test depth uniform over test graphs}
\label{app:uniform-extrapolation}

Fix the training law, $N$, and $K\ge1$.
To select a depth without inspecting the test graph, we replace \(s_{\max}(G)\) by its universal bound \(N\) in the upper endpoint. Thus \(\underline T_K\) is a common lower bound on all graph-specific upper endpoints, and \(T_K^\circ\) is the largest integer strictly below this common bound.
Suppose $\delta_K>0$ and define
\begin{equation}
\label{eq:uniform-upper}
\underline T_K
:=\log_3\!\left[
1+\frac{(3^K-1)c_{\rm out}}
{\log(1+\delta_K/\kappa)}
\right],
\qquad
T_K^\circ:=\lceil\underline T_K\rceil-1.
\end{equation}
These quantities depend only on the training problem. The integer $T_K^\circ$ is nonnegative because
$\underline T_K>0$.

\begin{lemma}[A common inference depth]
\label{lem:uniform-test-depth}
At the population training optimum, the same depth $T_K^\circ$
recovers every connectivity label of every graph $G\in\mathcal G_N$
with $\mathsf{Hard}(G)<T_K^\circ$.
If $T_K^\circ>K$, this is depth extrapolation.
\end{lemma}

\begin{proof}
For every $G\in\mathcal G_N$, the bound $s_{\max}(G)\le N$ gives
\[
\log\!\left(1+\frac{\tau N}{(1-\tau)s_{\max}(G)}\right)
\ge c_{\rm out}.
\]
Consequently $T_{\max}(G;K)\ge\underline T_K$ for all test
graphs simultaneously. If $\mathsf{Hard}(G)<T_K^\circ$, then
\begin{equation}
\label{eq:uniform-solvability-chain}
\mathsf{Hard}(G)<T_K^\circ<\underline T_K
\le T_{\max}(G;K).
\end{equation}
Proposition~\ref{prop:test-window} now applies to this same graph
at this same integer depth. Its connected and degenerate-spectrum
cases also satisfy the conclusion.
\end{proof}

The rounding relation
\begin{equation}
\label{eq:hardness-cutoff}
\underline T_K-1\le T_K^\circ<\underline T_K
\end{equation}
changes the common sufficient depth by at most one. When $\delta_K=0$, any chosen integer $T$ works for all
graphs satisfying $\mathsf{Hard}(G)<T$.

\subsection{Monotonicity with training depth}
\label{app:proof:prop:monotonicity}

\begin{proposition}[More training loops preserve the guarantees]
\label{prop:monotonicity}
For a fixed training law and its population optima, increasing $K$
makes $\delta_K$ nonincreasing and $\gamma_*$ nondecreasing,
and preserves every test depth certified by
Proposition~\ref{prop:test-window}.
\end{proposition}

\begin{proof}
Every non-unit eigenvalue lies in $[0,1)$, so the spectral expression
in~\eqref{eq:statistics} is nonincreasing in $K$.
Also
\[
-\log\gamma_*=
\frac{2\log(1+\delta_K/\kappa)}{3^K-1}
\]
is nonincreasing. For each test graph, the lower endpoint is
independent of $K$ and the upper endpoint is nondecreasing.
The same inclusion holds after restricting to integer test depths.
The common depths in~\eqref{eq:uniform-upper} are also
nondecreasing in $K$ at the population optimum.
\end{proof}

\subsection{Proof of Theorem~\ref{thm:hardness-gap}}
\label{app:hardness-gap-proof}

We first state the regularity condition used in
Theorem~\ref{thm:hardness-gap}.

\begin{definition}[Spectrally regular training family]
\label{def:spectral-regularity}
Let $\mathcal D_j$ be exchangeable distributions over
$\mathcal G_{N_j}$ with training depths $K_j\ge1$,
$\kappa_j>0$, and $\delta_j>0$. Write
\[
H_j=\mathsf{Hard}_{\rm tr}(\mathcal D_j),
\qquad
r_j=-\log\rho_{{\rm tr},j},
\]
Here $H_j$ is the maximum hardness over graphs of positive training
probability, $\rho_{{\rm tr},j}\in(0,1)$ is the largest nonunit
eigenvalue over the training support, and $\delta_j$ is
the spectral residual at depth $K_j$. The maxima defining \(H_j\) and \(\rho_{\mathrm{tr},j}\) need not be attained by the same graph; the first regularity condition supplies the comparison used below.

The family is spectrally regular at these training depths
if there exist constants
$0<a_-\le a_+<\infty$ and
$0<z_-\le z_+<\infty$, independent of $j$, such that
\[
a_-H_j
\le
r_j3^{H_j}
\le
a_+H_j,
\]
and
\[
z_-
\le
\frac{
e^{2r_j3^{K_j}}\delta_j
}{
\kappa_j
}
\le
z_+.
\]
\end{definition}

The first condition connects training hardness to the slowest
nontrivial spectral mode. The second requires this mode to control
the residual training error up to constant factors.

\begin{proof}[Proof of Theorem~\ref{thm:hardness-gap}]
Spectral regularity makes the normalized training residual exponentially small in \(R=H3^{K-H}\). The common depth bound depends logarithmically on the inverse residual, producing an additional depth of order \(R\). Lemma 2 then turns this depth into a simultaneous guarantee for all test graphs below its hardness cutoff. We suppress the family index \(j\) and write \(\delta=\delta_K\) below.

\paragraph{1. Training controls a graph-independent depth.}
Put $R=H3^{K-H}$. By spectral regularity,
\[
a_-R\le r3^K\le a_+R,
\qquad
z_-e^{-2r3^K}\le\frac\delta\kappa
\le z_+e^{-2r3^K}.
\]
Since $R\to\infty$, $\delta/\kappa\to0$ and
\begin{equation}
\label{eq:residual-log-rate}
-\log\log(1+\delta/\kappa)=2r3^K+O(1).
\end{equation}
The $O(1)$ is uniform: the multiplicative residual factor is between
$z_-$ and $z_+$, and $\log(1+x)/x\to1$.
Using~\eqref{eq:uniform-upper} and
$\log_3(3^K-1)=K+O(1)$ uniformly for $K\ge1$ gives
\begin{equation}
\label{eq:tmax-sharp}
\underline T_K
=K+\frac{2r3^K}{\log3}+O(1),
\qquad
T_K^\circ
=K+\frac{2r3^K}{\log3}+O(1).
\end{equation}
In particular, $T_K^\circ-K=\Theta(R)\to\infty$.

\paragraph{2. Verify the common test window.}
Choose $c=a_-/\log3>0$. For sufficiently large indices,
\[
T_K^\circ\ge K+\frac{2a_-}{\log3}R-O(1)
>K+cR>K.
\]
For an arbitrary $G\in\mathcal G_N$ satisfying
$\mathsf{Hard}(G)<T_K^\circ$, Lemma~\ref{lem:uniform-test-depth} applies. In particular, for the smaller cutoff $\mathsf{Hard}(G)<K+cR$, we have
\[
\max\{K,\mathsf{Hard}(G)\}<T_K^\circ
<\underline T_K\le T_{\max}(G;K).
\]
Lemma~\ref{lem:uniform-test-depth}, or directly
Proposition~\ref{prop:test-window}, gives exact connectivity
recovery at $T_K^\circ$. The chosen depth depends only on the
training problem, so the conclusion holds simultaneously for all
such test graphs.

The constants otherwise depend only on $a_-,a_+,z_-,z_+$.
\end{proof}

Theorem~\ref{thm:hardness-gap} uses the lower bound in
\eqref{eq:tmax-sharp} to certify test graphs. The path example below
realizes this guarantee with growing train and test graph families.

\subsection{High-probability shallow-to-deep extrapolation}
\label{app:probabilistic-extrapolation}

\begin{proof}[Proof of Theorem~\ref{thm:probabilistic-extrapolation}]
Work along a sequence of training problems satisfying
Theorem~\ref{thm:gf}.
All limits below are joint in $N,K,H,\varepsilon$.
Let $\mathcal A_H=\{G:\mathsf{Hard}(G)\le H\}$, so that
$\Pr_{\mathcal D_{\rm tr}}(\mathcal A_H^c)\le\varepsilon$.

\paragraph{1. Bound the population residual.}
For a nondegenerate graph in $\mathcal A_H$, the definition of
hardness implies
\[
-\log\rho(G)
\ge 3^{-H}\bigl[\log s_{\max}(G)+c_{\rm in}\bigr]
\ge c_{\rm in}3^{-H}.
\]
Its squared spectral residual is therefore at most
$N\exp(-2c_{\rm in}3^{K-H})$. Degenerate graphs have residual zero.
On $\mathcal A_H^c$, it is at most $N$. Taking expectations gives
\begin{equation}
\label{eq:delta-quantile}
\delta_K\le N\bigl[e^{-2c_{\rm in}3^{K-H}}+\varepsilon\bigr].
\end{equation}
Let $M=\min\{3^{K-H},\log(1/\varepsilon)\}$ and
$\beta=\min\{2c_{\rm in},1\}>0$. Then
\begin{equation}
\label{eq:delta-quantile-exp}
\delta_K\le2N e^{-\beta M},
\end{equation}
and $\log(1+x)\le x$ implies
\begin{equation}
\label{eq:log-delta-quantile}
\log(1+\delta_K/\kappa)
\le\frac{2N}{\kappa}e^{-\beta M}.
\end{equation}

\paragraph{2. Obtain a common upper-endpoint lower bound.}
Suppose first $\delta_K>0$. Substitution into
\eqref{eq:uniform-upper} yields
\begin{equation}
\label{eq:tmax-quantile}
\underline T_K\ge
\log_3\!\left[
\frac{(3^K-1)\kappa c_{\rm out}}{2N}
 e^{\beta M}\right].
\end{equation}
Consequently
\begin{equation}
\label{eq:tmax-quantile-rate}
\underline T_K\ge K+\frac{\beta}{\log3}M
-\log_3(N/\kappa)-C,
\end{equation}
for a fixed finite constant $C$.
Because $M\to\infty$ and $\log(N/\kappa)=o(M)$, eventually
\[
\underline T_K\ge K+\frac{\beta}{2\log3}M.
\]
The constant is uniform over all test graphs on $N$ vertices.

\paragraph{3. Choose one integer depth and verify both sides.}
Set $c=\beta/(4\log3)$, $R_K=K+c M$, and
$T_K^{\rm hp}=\lfloor R_K\rfloor+1$.
For sufficiently large indices, $c M>1$, so
\[
K<R_K<T_K^{\rm hp}\le R_K+1
<K+2c M\le\underline T_K.
\]
Every test graph satisfying $\mathsf{Hard}(G)<R_K$ therefore obeys
\begin{equation}
\label{eq:hp-solvability-chain}
\mathsf{Hard}(G)<R_K<T_K^{\rm hp}
<\underline T_K\le T_{\max}(G;K).
\end{equation}
More generally, every graph with $\mathsf{Hard}(G)<T_K^{\rm hp}$ satisfies the same window directly, proving the stated cutoff. Proposition~\ref{prop:test-window} gives exact recovery at this
common integer depth. If $\delta_K=0$, the upper restriction is
absent and the same choice of $T_K^{\rm hp}$ works directly.

Finally, for any test law $\mathcal D_{\rm te}$ on
$\mathcal G_N$, the deterministic inclusion just proved gives
\[
\Pr_{G\sim\mathcal D_{\rm te}}\!\bigl(
\widehat Y(T_K^{\rm hp})=Y_G\bigr)
\ge
\Pr_{G\sim\mathcal D_{\rm te}}\!\bigl(\mathsf{Hard}(G)<R_K\bigr).
\]
Thus mass at least $1-\zeta$ below $R_K$ gives exact-match accuracy
at least $1-\zeta$ at the training-selected depth. The training tail
probability controls the population residual.
\end{proof}

\subsection{Short to Long Path Generalization}
\label{app:spectral-refinement}

We prove Example~\ref{ex:path-generalization} in three steps: first
verify the assumptions of Theorem~\ref{thm:hardness-gap} for the
spectrally regular packed-path families, then retain the constants
in its window argument, and finally recover the path-length
statement $1<q<1+3^\Delta$ in the main text. The refinement also
covers packed-path families whose leftover isolated vertices do
not satisfy the regularity assumption.

\begin{proof}[Proof of Example~\ref{ex:path-generalization}]\leavevmode\par
\paragraph{1. Verify the theorem's assumptions.}

Let $P_v$ denote a path on $v$ vertices and define
\[
G_{N,v}=\lfloor N/v\rfloor P_v
\sqcup\bigl(N-v\lfloor N/v\rfloor\bigr)P_1.
\]
Fix $s\ge3$, $N\ge2s$, and $K\ge1$.
Train on uniform vertex relabelings of $G_{N,s}$ under the
assumptions of Theorem~\ref{thm:gf}, and evaluate the population-optimal
block on $G_{N,L}$, $2\le L\le N$.
Write $\ell=\lfloor N/s\rfloor$ and $b=N-\ell s$.
For a path component of size $v\ge2$ under the global normalization
$I-\mathcal L/4$, define
\begin{equation}
\label{eq:path-example-spectrum}
\lambda_{v,j}=\cos^2\!\left(\frac{\pi j}{2v}\right),\qquad
R_v(K)=\sum_{j=1}^{v-1}\lambda_{v,j}^{\,2\cdot3^K},
\end{equation}
and set $R_0(K)=R_1(K)=0$.
The training graph contains an $s$-vertex path, so its maximum degree
is two. Its exact normalized residual is
\begin{equation}
\label{eq:path-example-training-residual}
\kappa=\ell+b-1,\qquad
D_{N,s,K}:=\delta_K/\kappa
=\frac{\ell R_s(K)}{\ell+b-1}\le2R_s(K).
\end{equation}
The inequality follows from $\ell\ge2$.

For $L\ge3$, the largest component and slowest non-unit mode give
\[
h(L):=\mathsf{Hard}(G_{N,L})
=\log_3\!\left(
\frac{\log L+c_{\rm in}}{-\log\lambda_{L,1}}
\right).
\]
This quantity is independent of the number of path components and
strictly increases with $L$. Set $h(2)=-\infty$: when $L=2$,
the test graph has maximum degree one and its actual input operator
already equals its component projector.

Let $H_s=h(s)$, $r_s=-\log\lambda_{s,1}$, and
$t_s=r_s3^{K_s}$. For the regime in the example,
$K_s-H_s\to\Delta\in\mathbb R$. The path spectrum gives
\[
r_s\sim\frac{\pi^2}{4s^2},\qquad
r_s3^{H_s}=\log s+c_{\rm in},\qquad
H_s\sim2\log_3s.
\]
Thus $H_s\to\infty$, $H_s3^{K_s-H_s}\to\infty$, and
$r_s3^{H_s}/H_s\to(\log3)/2$, verifying the first spectral
regularity condition. Also $t_s\sim3^\Delta\log s\to\infty$.
For $r_{s,j}=-\log\lambda_{s,j}$, monotonicity of $\tan x/x$
implies $r_{s,j}\ge j^2r_s$. Consequently,
\begin{equation}
\label{eq:path-example-mode-dominance}
e^{-2t_s}\le R_s(K_s)
\le\frac{e^{-2t_s}}{1-e^{-6t_s}},
\end{equation}
since $j^2-1\ge3(j-1)$. Hence
\[
e^{2t_s}\frac{\delta_{K_s}}\kappa
=\frac{\ell}{\ell+b-1}(1+o(1)).
\]
If $b=O(\ell)$, this ratio is bounded above and away from zero,
which verifies the second spectral regularity condition. This
includes $s\mid N$, where $b=0$. Uniform relabeling makes the
training law exchangeable, and $\kappa=\ell+b-1>0$; with the
positive initialization in Theorem~\ref{thm:gf}, all the remaining
training assumptions hold. Theorem~\ref{thm:hardness-gap} therefore
gives a common depth with $T-K=\Theta(H_s3^{K-H_s})$ for this
regular subclass.

The condition $b=O(\ell)$ is needed for this invocation: if the
isolated remainder dominates the number of paths, the normalized
residual need not satisfy the required lower bound. The sharper
argument below uses only $\delta_K/\kappa\le2R_s(K)$, so it applies
to every $N\ge2s$ and establishes the full example without that
extra restriction.

\paragraph{2. Refine the common depth, retaining the constants.}
The \(\Theta\)-rate in Theorem 2 does not determine the exponent \(1+3^\Delta\) in Example 1. We therefore retain the leading constant in the depth bound before translating depth into path length.
Define
\begin{equation}
\label{eq:packed-path-depth}
\begin{aligned}
\bar B_{s,K}
&=\log_3\!\left[
1+\frac{(3^K-1)c_{\rm out}}{\log(1+2R_s(K))}
\right],\\
\bar T_{s,K}&=\lceil\bar B_{s,K}\rceil-1,
\end{aligned}
\end{equation}
and
\begin{equation}
\label{eq:packed-path-cap}
L_{\rm cert}(N,s,K)
=\max\{2\le L\le N:h(L)<\bar T_{s,K}\},
\end{equation}
where the maximum is over integers. The learned block recovers every
label of every $G_{N,L}$ with $2\le L\le L_{\rm cert}(N,s,K)$,
and every relabeling, at the same depth $\bar T_{s,K}$.

Let $s\to\infty$, $N_s\ge2s$, $r_s=-\log\lambda_{s,1}$,
and $t_s=r_s3^{K_s}$. If $t_s\to\infty$, then
\begin{equation}
\label{eq:path-example-uniform-depth}
\bar T_{s,K_s}=K_s+\frac{2t_s}{\log3}+O(1),
\end{equation}
and, uniformly over $N_s$,
\begin{equation}
\label{eq:packed-path-length-rate}
L_{\rm cert}(N_s,s,K_s)
=\Theta\!\left(
\min\left\{N_s,
\frac{3^{K_s/2}e^{t_s}}{\sqrt{K_s+t_s}}\right\}\right).
\end{equation}
We establish these bounds directly from the same window condition
used in Theorem~\ref{thm:hardness-gap}.
For every feasible $L$, its largest component has size $L\le N$.
Consequently,
\[
\log\!\left(1+\frac{\tau N}{(1-\tau)L}\right)
\ge c_{\rm out}.
\]
Combining this inequality with
$D_{N,s,K}\le2R_s(K)$ and Proposition~\ref{prop:test-window} gives
\begin{equation}
\label{eq:packed-path-window-chain}
h(L)<\bar T_{s,K}<\bar B_{s,K}
\le T_{\max}(G_{N,L};K)
\end{equation}
whenever $h(L)<\bar T_{s,K}$ and the test graph is disconnected.
For a connected test graph, the upper restriction is unnecessary.
The case $L=2$ has no lower restriction; the upper bound remains
valid. This proves the finite guarantee, including relabelings by
permutation equivariance.

For the asymptotic refinement, \eqref{eq:path-example-mode-dominance} gives
\begin{equation}
\label{eq:path-example-log-residual}
-\log\log(1+2R_s(K_s))=2t_s+O(1).
\end{equation}
Substitution into~\eqref{eq:packed-path-depth} proves
\eqref{eq:path-example-uniform-depth}, including integer rounding.

Taylor expansion yields
\begin{equation}
\label{eq:path-example-hardness}
\begin{aligned}
-\log\lambda_{L,1}&=\frac{\pi^2}{4L^2}(1+O(L^{-2})),\\
h(L)&=2\log_3L+\log_3\log L
+\log_3(4/\pi^2)+o(1).
\end{aligned}
\end{equation}
If $\Lambda(T)=\max\{L\ge2:h(L)<T\}$ over integers,
inverting this expansion gives
\[
\Lambda(T)
=\frac{\pi}{\sqrt{2\log3}}\frac{3^{T/2}}{\sqrt T}(1+o(1)).
\]
Since $L_{\rm cert}(N,s,K)=\min\{N,\Lambda(\bar T_{s,K})\}$,
\eqref{eq:packed-path-length-rate} follows.

\paragraph{3. Recover the path-length statement in the main text.}
Suppose $s\to\infty$, $N_s\ge2s$, and
$K_s-H_s\to\Delta\in\mathbb R$, where $H_s=h(s)$.
Then $\bar T_{s,K_s}>K_s$ eventually and
\begin{equation}
\label{eq:packed-path-exponent}
\frac{\log L_{\rm cert}(N_s,s,K_s)}{\log s}
=\min\left\{\frac{\log N_s}{\log s},1+3^\Delta\right\}+o(1).
\end{equation}
In particular, for any fixed $1<q<1+3^\Delta$, the same depth
solves every feasible integer size
$2\le L\le\min\{N_s,\lfloor s^q\rfloor\}$ for all sufficiently
large $s$. It also solves every feasible sequence
$L_s=s^{q+o(1)}\le N_s$ with this strict exponent gap.
To see this explicitly, write $A_s=\log s+c_{\rm in}$. Exactly,
$r_s3^{H_s}=A_s$, so
$t_s=A_s3^{K_s-H_s}\sim3^\Delta\log s$ and
$K_s=(2/\log3)\log s+o(\log s)$.
The depth and length formulas above imply the exponent and
$\bar T_{s,K_s}-K_s\to\infty$. For any fixed $q$ below the stated
boundary, $h(s^{q+o(1)})/H_s\to q$, whereas
$\bar T_{s,K_s}/H_s\to1+3^\Delta$. The strict lower-window condition
therefore holds eventually. Taking $T=\bar T_{s,K_s}$, we have shown
exactly the implication stated in the main text:
\[
\begin{gathered}
K_s-H_s\to\Delta,\qquad L_s=s^{q+o(1)}\le N_s,\qquad 1<q<1+3^\Delta
\\[2pt]
\Longrightarrow\quad
\widehat Y_{G_{N_s,L_s}}(T)=Y_{G_{N_s,L_s}}\quad\text{at one common }T>K_s.
\end{gathered}
\]
Limits are along admissible integer
sequences. Although $N$ is fixed within each comparison, any limit
$s\to\infty$ with $N_s\ge2s$ necessarily has $N_s\to\infty$.
\end{proof}

\subsection{ER-to-ER generalization}
\label{app:er-generalization}

The proof separates two requirements. A bound on the training residual determines one depth \(T_K\) that controls false connections for every test graph. A component-size bound then identifies the graphs for which this depth provides sufficient propagation. Combining this size budget with the ER component tail yields the density interval in Example 2. Training and testing
use the same $n\ge3$ vertices, and $0<c_{\rm tr}:=np_{\rm tr}<1$.

\begin{proof}[Proof of Example~\ref{ex:er-generalization}]\leavevmode\par
\paragraph{1. Verify the theorem's assumptions.}
The ER law is invariant under vertex relabeling. Moreover, the
lower bound on $\kappa$ below is strictly positive. Thus the
exchangeability and nondegeneracy conditions of
Theorem~\ref{thm:gf} hold; we retain its positive initialization
assumption and evaluate the population-optimal block.
To verify the hardness-tail assumption, we first relate component
size to spectral decay.

\emph{Component-size tail.} For $c\in(0,1)$, write $\alpha(c)=c-1-\log c$.
If $G\sim G(n,c/n)$, then for every integer $S\ge1$,
\begin{equation}
\label{eq:er-example-size-tail}
\Pr(s_{\max}(G)>S)\le n e^{-S\alpha(c)}.
\end{equation}
To prove this, explore the component of one vertex. The number
of new vertices found when processing each explored vertex is
dominated by an independent $\operatorname{Bin}(n,c/n)$ variable,
using independent padding variables for unqueried positions.
A component larger than $S$ requires at least $S$ discoveries
while its first $S$ vertices are processed. For the dominating
variables $Z_1,\ldots,Z_S$ and any $t>0$,
\[
\Pr\!\left(\sum_{i=1}^{S}Z_i\ge S\right)
\le
\exp\!\left[S\{-t+c(e^t-1)\}\right].
\]
Choosing $t=-\log c$ gives $e^{-S\alpha(c)}$. A union bound over
the root proves~\eqref{eq:er-example-size-tail}. For $S\ge n$ the
bound holds immediately.

\emph{From component size to hardness.}
If every component of an unweighted graph has size at most $S\ge2$,
then, for the normalization used in this paper,
\begin{equation}
\label{eq:er-example-gap}
\|P_G-\Pi_G\|_{\rm op}
\le 1-\frac{2}{S(S-1)^2}.
\end{equation}
Indeed, on a component with $s\ge2$ vertices, a mean-zero vector
$v$ with range $R$ satisfies
$\|v\|_2^2\le sR^2/4$. A path of length at most $s-1$ joining
its minimum and maximum entries gives
\[
v^\top L_Gv\ge\frac{R^2}{s-1}.
\]
Thus the smallest positive component Laplacian eigenvalue is
at least $4/[s(s-1)]$. The global maximum degree is at most
$S-1$, proving~\eqref{eq:er-example-gap} after normalization.
Consequently, for every integer $m\ge1$,
\begin{equation}
\label{eq:er-example-residual}
\|P_G^m-\Pi_G\|_{\rm op}
\le
\exp\!\left[-\frac{2m}{S(S-1)^2}\right].
\end{equation}
Singleton components contribute no non-unit modes.

Define, for integer $S\ge1$,
\begin{equation}
\label{eq:er-example-F}
F(S):=
\frac{S(S-1)^2}{2}
\bigl(\log S+c_{\rm in}\bigr).
\end{equation}
In particular $F(1)=0$. The within-component condition
in Proposition~\ref{prop:test-window} holds for every graph
with largest component at most $S$ whenever $F(S)<3^T$.
For $S=1$, within-component correctness is automatic.

\emph{Training cutoff and the remaining asymptotic conditions.}
Put $c_{\rm tr}=np_{\rm tr}$ and define
\begin{equation}
\label{eq:er-example-training-size}
S_{\rm tr}:=
\min\!\left\{
n,\,
\max\!\left\{
2,\left\lceil
\frac{\log(n/\varepsilon)}{\alpha(c_{\rm tr})}
\right\rceil
\right\}
\right\}.
\end{equation}
By~\eqref{eq:er-example-size-tail},
$\Pr(s_{\max}(G_{\rm tr})>S_{\rm tr})\le\varepsilon$.
If the definition truncates at $n$, the exceptional event is empty.

For every graph, $c_G\ge n-|E(G)|$. Hence
\begin{equation}
\label{eq:er-example-kappa}
\begin{aligned}
\kappa
&\ge \underline\kappa_{\rm tr},\\
\underline\kappa_{\rm tr}
&:=(n-1)\left(1-\frac{c_{\rm tr}}2\right)>0.
\end{aligned}
\end{equation}
In particular $n/\underline\kappa_{\rm tr}\le3$ for $n\ge3$.
Set $H_{\rm tr}=\log_3 F(S_{\rm tr})$. For nondegenerate
graphs with $s_{\max}\le S_{\rm tr}$, the spectral bound implies
$\mathsf{Hard}(G)\le H_{\rm tr}$; degenerate graphs impose no
lower-depth restriction. Therefore
\[
\Pr_{\mathcal D_{\rm tr}}
 (\mathsf{Hard}(G)>H_{\rm tr})\le\varepsilon.
\]
Along any sequence for which
\[
M_{\rm thm}:=\min\left\{
\frac{3^K}{F(S_{\rm tr})},\log\frac1\varepsilon
\right\}\longrightarrow\infty,
\]
the last condition of Theorem~\ref{thm:probabilistic-extrapolation}
also holds: $\log(n/\kappa)\le\log3=o(M_{\rm thm})$.
The theorem consequently supplies one training-selected depth
$T>K$ with $T-K=\Theta(M_{\rm thm})$, solving all test graphs
with $\mathsf{Hard}(G)<T$. This invocation is asymptotic; it does
not assert $T>K$ for arbitrary fixed $n,K,\varepsilon$.

\paragraph{2. Refine the guarantee at finite graph size.}
We now retain the explicit component-size decay in the theorem's
residual argument. This gives a computable depth and avoids hiding
constants in the asymptotic statement.
Splitting the spectral residual according to the event
$s_{\max}\le S_{\rm tr}$ and using at most $n$ non-unit modes,
\[
\delta_K
\le
n\left\{
\exp\!\left[
-\frac{4\,3^K}{S_{\rm tr}(S_{\rm tr}-1)^2}
\right]
+\varepsilon
\right\}.
\]
Therefore the explicitly computable quantity
\begin{equation}
\label{eq:er-example-U}
\begin{aligned}
U_K:=\frac{n}{\underline\kappa_{\rm tr}}\Bigg\{
&\exp\!\left[
-\frac{4\,3^K}{S_{\rm tr}(S_{\rm tr}-1)^2}
\right]
&+\varepsilon
\Bigg\}
\end{aligned}
\end{equation}
satisfies $\delta_K/\kappa\le U_K$ and is explicit in the
training-distribution parameters.

\emph{Choose one depth and its component-size budget.}
Define
\begin{equation}
\label{eq:er-example-common-depth}
T_K:=
\left\lceil
\log_3\!\left[
1+\frac{(3^K-1)c_{\rm out}}{\log(1+U_K)}
\right]
\right\rceil-1.
\end{equation}
This is a nonnegative integer. For any test graph with
$s_{\max}\le n$,
\[
\log\!\left(
1+\frac{\tau n}{(1-\tau)s_{\max}}
\right)\ge c_{\rm out}.
\]
Together with $\delta_K/\kappa\le U_K$, this proves
$T_K<T_{\max}(G;K)$ whenever the upper restriction is relevant.
The ceiling minus one handles strictness even when the continuous
upper endpoint is an integer.

Define the integer size budget
\begin{equation}
\label{eq:er-example-size-budget}
S_K:=
\max\{S\in\{1,\ldots,n\}:F(S)<3^{T_K}\}.
\end{equation}
The set is nonempty because $F(1)=0$.
For each nondegenerate graph with $s_{\max}\le S_K$, the
size bound controls that graph's own spectral error, giving
\[
\mathsf{Hard}(G)<T_K<T_{\max}(G;K).
\]
Degenerate graphs have no lower restriction. Thus
\eqref{eq:er-example-common-depth} and
\eqref{eq:er-example-size-budget} imply exact recovery for
\emph{every} such graph at the same depth $T_K$.
For extrapolation strictly beyond training, we additionally
require $T_K>K$.

\paragraph{3. Recover the density interval in the main text.}
The function $\alpha$ is strictly decreasing from $+\infty$ to zero
on $(0,1)$, since $\alpha'(c)=1-1/c<0$.
For a subcritical test graph $G_{\rm te}\sim G(n,p_{\rm te})$,
\eqref{eq:er-example-size-tail} gives
\[
\Pr(s_{\max}(G_{\rm te})>S_K)
\le
n e^{-S_K\alpha(np_{\rm te})}.
\]
The displayed upper bound is at most $\varepsilon$ exactly when
\[
p_{\rm te}\le
\frac1n\alpha^{-1}\!\left(
\frac{\log(n/\varepsilon)}{S_K}
\right).
\]
This proves Example~\ref{ex:er-generalization}, with probability
taken over the sampled test graph.
The same certificate also applies to $p_{\rm te}=0$.

The denser-than-training interval in the example is nonempty
precisely when
\[
S_K>
\frac{\log(n/\varepsilon)}{\alpha(np_{\rm tr})}.
\]
Together with $T_K>K$, this condition gives extrapolation to
denser test graphs.

\end{proof}

\paragraph{Relation between the refinement and the theorem.}
The construction applies at finite graph size. To connect
it to the high-probability theorem, let
\[
M_K:=
\min\left\{
\frac{4\,3^K}{S_{\rm tr}(S_{\rm tr}-1)^2},
\log\frac1\varepsilon
\right\}.
\]
If $M_K\to\infty$, then $-\log U_K=M_K+O(1)$, uniformly in
$n$ and $c_{\rm tr}\in(0,1)$, and
\[
T_K=K+\frac{M_K}{\log3}+O(1).
\]
The component-size bound gives this scaling directly from the
ER distribution parameters. In contrast, the direct theorem
invocation uses $M_{\rm thm}$ above. Since
$F(S)$ also contains a logarithmic factor, these two rates
need not coincide: retaining the explicit spectral decay gives a
sharper certificate.

Before the cap $n$ is active, as $T_K\to\infty$,
\[
S_K=
\left(\frac{6\,3^{T_K}}{T_K\log3}\right)^{1/3}(1+o(1)).
\]
If $x_K:=\log(n/\varepsilon)/S_K\to0$, then the certified
density endpoint $p_K^{\rm bd}:=n^{-1}\alpha^{-1}(x_K)$ satisfies
\[
1-np_K^{\rm bd}\sim\sqrt{2x_K}.
\]
This expansion of the sufficient density bound follows from
$\alpha(1-z)=z^2/2+O(z^3)$.
\section{Additional Empirical Settings}
\label{app:empirical-checks}

\subsection{Population convergence}

Figure~\ref{fig:population-convergence} checks convergence on uniform
relabelings of $C_{10}\sqcup C_{10}$, with $N=20$, $K=3$, and
constant learning rate $0.002$. Four positive invariant initializations approach the same optimal ratio.

\begin{figure}[!htb]
\centering
\includegraphics[width=0.92\textwidth]{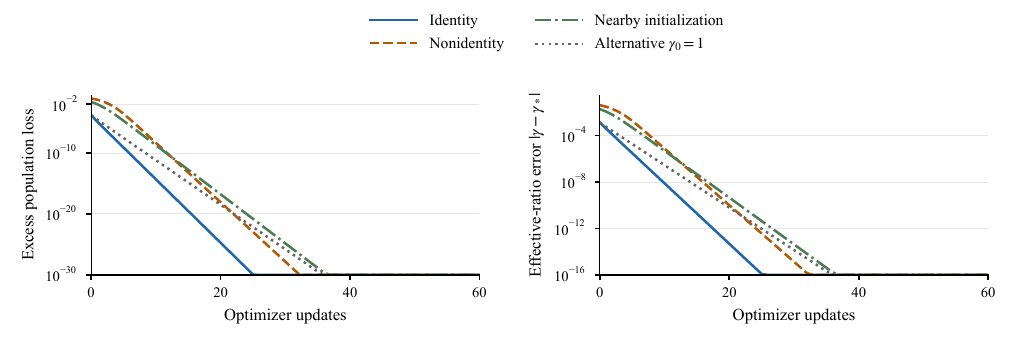}
\caption{Population GD from four initializations. Both excess loss
and error in $\gamma$ approach zero at training depth $K=3$.}
\label{fig:population-convergence}
\end{figure}

\subsection{Learning rates and finite-budget optimization}
\label{app:empirical-optimization}
\label{app:lr-sweep}

The learning-rate sweep in Figure~\ref{fig:lr-sweep} supports the
optimization comparison in Figure~\ref{fig:flop-matched}. Several
rates converge for $K\le4$. At $K=6,8$, none of the tested constant
rates converges within the budget: increasing the rate can cause
divergence. The useful learning-rate range depends on training depth.

\begin{figure}[!htb]
\centering
\includegraphics[width=0.6\textwidth]{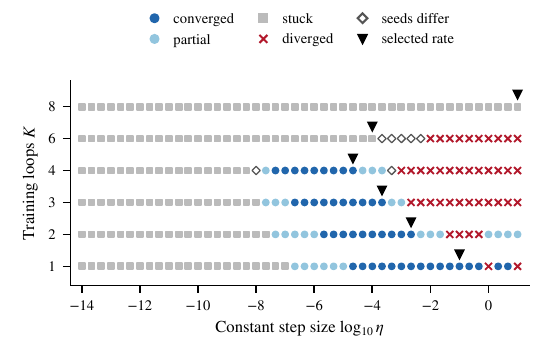}
\caption{Outcomes of the constant-learning-rate sweep at the largest
budget. Diamonds indicate outcomes that vary across runs; triangles mark
the rates used in Figure~\ref{fig:flop-matched}.}
\label{fig:lr-sweep}
\end{figure}

\subsection{Initialization and useful training depth}
\label{app:empirical-compute}

Figure~\ref{fig:compute-tradeoff} adds untrained baselines to
Figure~\ref{fig:choose-k}. Near-identity initialization already solves
harder graphs before training. The trained curves should therefore
be compared with their respective baselines when assessing the
benefit of additional training loops.

\begin{figure}[!htb]
\centering
\includegraphics[width=0.85\textwidth]{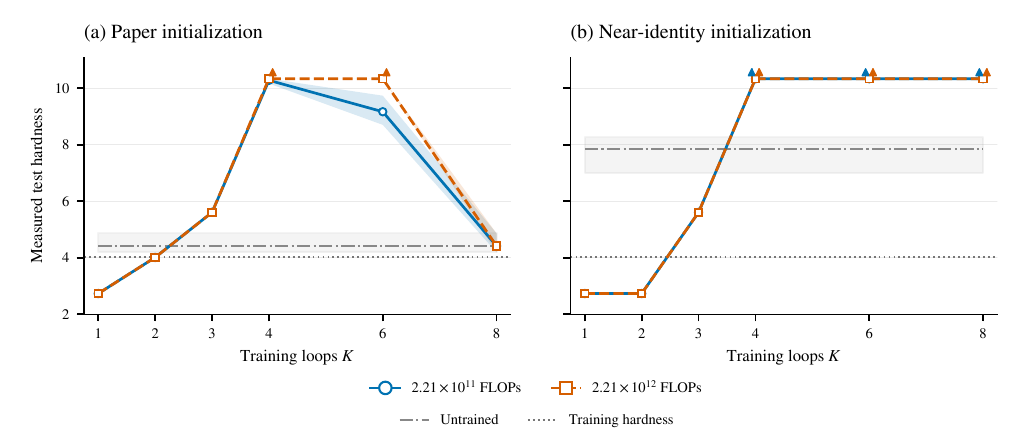}
\caption{Full-matrix Adam training at matched FLOPs. Curves and bands
show mean solved hardness and ranges over three runs. Dash-dotted
lines show untrained performance; the dotted line marks training
hardness. Arrows indicate success on the entire test catalog.}
\label{fig:compute-tradeoff}
\end{figure}

\section{Figure Protocols}
\label{app:figure-protocol}

\subsection{Population training and exact recovery}

The population experiments use full-batch gradient descent in the
Frobenius-isometric coordinates
$x=(\sqrt{N-1}a,b,\sqrt{N-1}c,d)$ of
\eqref{eq:rescaled-coordinate}. The training samples are averaged
with an exactly exchangeable training law, so the full-batch loss
is invariant under every vertex relabeling. From the symmetric
initialization, its full matrix gradient therefore remains in the
$E$--$F$ subspace at every update. Consequently, with the same
learning rate and initialization, four-coordinate GD and full-matrix
GD produce exactly the same matrix iterates in exact arithmetic
(Proposition~\ref{prop:invariance} in
Appendix~\ref{app:proof:prop:invariance} and the Frobenius-coordinate
GD formulation in Appendix~\ref{app:proof:thm:gd}).
This is a discrete-time identity, not a gradient-flow approximation;
no minibatch sampling breaks the symmetry in these experiments.
The four-coordinate implementation is thus an exact computational
reduction of full-matrix GD. The separate Adam experiment below
updates every entry of $W,V$.
All experiments use whole-output row normalization and the loss
$\frac12\|H_K-\Pi_G\|_F^2$. Exact graph accuracy requires every
entry of the diagonal-normalized readout to match connectivity at
the fixed threshold.

Figure~\ref{fig:depth-overview} trains on uniform relabelings of
$8P_4$ from $(a,b,c,d)=(0.90,1.05,1.15,1.03)$ for 3000 updates
with learning rate $2\times10^{-4}$. Displayed predictions use the
literal matrix block through $T=12$.
Figure~\ref{fig:population-convergence} uses initial coefficients
$(1,1,1,1)$, $(0.90,1.05,1.15,1.03)$,
$(0.80,1.25,1.50,0.98)$, and $(0.70,1.30,1.30/0.70,1)$.
Each run takes 3000 updates; the plot shows the first 60.

\subsection{ER test-loop curves}
\label{app:er-protocol}

Figure~\ref{fig:ood-window-families} uses $N=128$ and 8,192 ER
training graphs with $Np_{\rm tr}=0.8$. Their uniform relabeling
orbits define one population shared by $K=2,3,4$.
GD starts from $W=V=I$ and uses 100-decimal-digit arithmetic
and Armijo backtracking,
stopping at relative error below $10^{-10}$ in $-\log\gamma$
against the stationary ratio.
Each test density $Np\in\{0.8,1.6,3\}$ has 512 independent graphs.
We evaluate $T=0,\ldots,32$ and display $T\le12$.
The lower panels give the fraction of graphs certified by
Proposition~\ref{prop:test-window}.
Forward evaluation uses float64 spectral powers and the learned
scalar recurrence. Training and test graphs are sampled independently.

\subsection{Path and ER transfer}
\label{app:transfer-protocol}

Figure~\ref{fig:ood-windows} starts from $W=V=I$ and uses
population GD with 4096-bit arithmetic and Armijo backtracking.
Training stops at relative error below $10^{-10}$ in
$-\log\gamma$ against its stationary value.

\paragraph{Paths.}
Fix $N=8192$. Training graphs contain $\lfloor N/s\rfloor$
copies of $P_s$ and isolated remainder nodes, with uniform relabelings.
The sizes $s=13,21,34,56,93$ approximate hardness levels
$h=5,6,7,8,9$. Train at $K=h-2,h-1,h,h+1$, plus a fixed-$K=5$
control: 22 distinct models in total.
The test family packs as many $L$-vertex paths as possible into
$N$ vertices, leaving the remainder isolated:
\[
G_{N,L}=\lfloor N/L\rfloor P_L
\sqcup\bigl(N-L\lfloor N/L\rfloor\bigr)P_1,
\qquad L=2,\ldots,4096.
\]
For $T=0,\ldots,32$, determine exact recovery from the
within-component scores and the maximum cross-component score.
Choose one common $T$ maximizing the
consecutive solved prefix beginning at $L=2$, breaking ties toward
smaller $T$. This uses test outcomes and permits $T\le K$.
Table~\ref{tab:path-scaling} retains the fixed-$K$ comparison:
increasing training hardness reduces $K-H$ when $K$ is fixed,
so the resulting trend can differ from the constant-gap asymptotics.

\begin{table}[htbp]
\centering
\caption{Solved path lengths at $N=8192$. An asterisk marks the
catalog limit $L=4096$; the last column holds training depth fixed.}
\label{tab:path-scaling}
\begin{tabular}{rrrrrrr}
\toprule
$s$ & $h$ & $K=h-2$ & $K=h-1$ & $K=h$ & $K=h+1$ & $K=5$ \\
\midrule
13 & 5 & 18 & 56 & 1000 & $4096^*$ & 1000 \\
21 & 6 & 32 & 96 & 1549 & $4096^*$ & 96 \\
34 & 7 & 56 & 245 & $4096^*$ & $4096^*$ & 56 \\
56 & 8 & 98 & 618 & $4096^*$ & $4096^*$ & 31 \\
93 & 9 & 173 & 1024 & $4096^*$ & $4096^*$ & 32 \\
\bottomrule
\end{tabular}
\end{table}

Evaluation uses the path eigenvectors and trained ratio, retaining
the uniform background on all $N$ vertices. Spectral tail bounds
control truncation error in the score matrices.

\paragraph{ER graphs.}
Set $N=64$, $K=1,2,3,4$, and
\[
Np_{\rm tr}\in\{0.1,0.2,\ldots,1,1.2,1.4,1.6,1.8,2,2.5,3,4\}.
\]
Each training population uses relabeling orbits of 128 sampled graphs.
The test grid is $Np_{\rm te}=0.05,0.10,\ldots,4$,
then $4.25,4.50,\ldots,8$, $9,10,\ldots,16$, and
$20,24,32,48,64$. Each density has 128 independent test graphs,
shared across models. The main plot shows $p_{\rm te}\le0.1$.
Each cell reports
\[
\max_{T\in\{0,\ldots,18\}}
\frac1{128}\sum_{i=1}^{128}
\mathbf1\{\widehat Y_{G_i}(T)=Y_{G_i}\}.
\]
Thus all graphs in a cell share one depth, selected using test
outcomes. The reported accuracies use one sampled test pool and
the directly evaluated depths $T=0,\ldots,18$.

\subsection{Optimization at matched compute}
\label{app:compute-protocol}

\paragraph{Population GD.}
Figure~\ref{fig:flop-matched} uses the uniform mixture of 64 fixed
$P_4/P_8$ partitions of $N=128$ vertices and 512-bit arithmetic.
For $K=1,2,3,4,6,8$, sweep $\eta=10^{j/3}$,
$j=-42,\ldots,3$, at the largest budget. Initial coefficients
$(0.90,1.05,1.15,1.03)$ receive independent multiplicative
log-jitter of at most $0.002$ for three runs, paired across $K$.
A run converges below excess loss $10^{-8}$, is stuck above half
its initial excess loss, and diverges upon leaving the positive
branch or exceeding loss $10^6$ times its initial value.
Choose the rate converging for all runs in the fewest mean updates;
if all fail to converge, choose the largest rate stable across all runs.
Selected rates for increasing $K$ are
$10^{-1},2.2\times10^{-3},2.2\times10^{-4},2.2\times10^{-5},
10^{-4},10$.
Budgets count $(24K-6)N^3$ leading matrix-multiplication FLOPs per
graph and update. This is the arithmetic cost of the full-matrix model.

\paragraph{Full-matrix Adam.}
Figures~\ref{fig:choose-k} and~\ref{fig:compute-tradeoff} train all
32,768 entries of $W,V\in\mathbb R^{128\times128}$ in float64.
Adam uses batch size 8, betas $(0.9,0.999)$, $\epsilon=10^{-8}$,
and zero weight decay. The training pool contains four relabelings
of each of 64 fixed partitions (256 graphs). Component sizes are
independently chosen as 4 or 8 with equal probability, using $P_4$
when four vertices remain. The pool and minibatch order are shared
across runs.
Initial coefficients are $(0.90,1.05,1.15,1.03)$ or
$(0.999,1,1,1)$, with paired multiplicative log-jitter bounded
by $0.002$ or $0.00005$, respectively. All matrix entries remain trainable.

Run $K=1,2,3,4,6,8$ with three random initializations and budgets
$B=2.21\times10^{11}$ and $10B$ GEMM FLOPs, rounded down to
whole updates at $8(24K-6)N^3$ matrix-multiplication FLOPs each.
An independent pilot run selects the learning rate for each $K$,
initialization family, and budget by endpoint loss on 256
held-out relabelings of the same partitions.
The initial grid is $\{1,3\}\times10^j$, $j=-9,\ldots,-5$,
extended when validation loss improves at a boundary. The $K=8$
paper initialization also tests rates through $10^{-3}$.
Ties favor the smaller rate.
Table~\ref{tab:adam-learning-rates} lists the choices.

\begin{table}[htbp]
\centering
\caption{Adam learning rates selected by an independent pilot run.}
\label{tab:adam-learning-rates}
\begin{tabular}{ccccc}
\toprule
& \multicolumn{2}{c}{$\gamma_0\approx0.96$}
& \multicolumn{2}{c}{$\gamma_0\approx0.999$} \\
\cmidrule(lr){2-3}\cmidrule(lr){4-5}
$K$ & $B$ & $10B$ & $B$ & $10B$ \\
\midrule
1 & $3\times10^{-5}$ & $10^{-6}$ & $3\times10^{-5}$ & $10^{-6}$ \\
2 & $10^{-6}$ & $10^{-7}$ & $3\times10^{-6}$ & $3\times10^{-7}$ \\
3 & $3\times10^{-6}$ & $3\times10^{-7}$ & $3\times10^{-7}$ & $3\times10^{-8}$ \\
4 & $3\times10^{-6}$ & $3\times10^{-7}$ & $10^{-7}$ & $10^{-8}$ \\
6 & $10^{-5}$ & $10^{-5}$ & $3\times10^{-7}$ & $3\times10^{-6}$ \\
8 & $10^{-9}$ & $10^{-9}$ & $10^{-6}$ & $10^{-6}$ \\
\bottomrule
\end{tabular}
\end{table}

The test catalog contains 13 path, seven broom, and four cycle
templates. Select one $T\in\{0,\ldots,16\}$ on 64 validation
relabelings per template, then evaluate 128 independent test
relabelings per template. Report the hardness of the longest prefix
in which every template reaches 95\% exact accuracy at that common
validation-selected depth. Bands show the range over three runs;
$10.34$ is the catalog's maximum hardness. Training uses an NVIDIA GH200.

\clearpage
\subsection{Numerical illustrations of the guarantees}

Figure~\ref{fig:hardness-illustrations} illustrates the certified
depth for growing path families and for a training distribution
with rare long paths.

\begin{figure}[!htb]
\centering
\includegraphics[width=\textwidth]{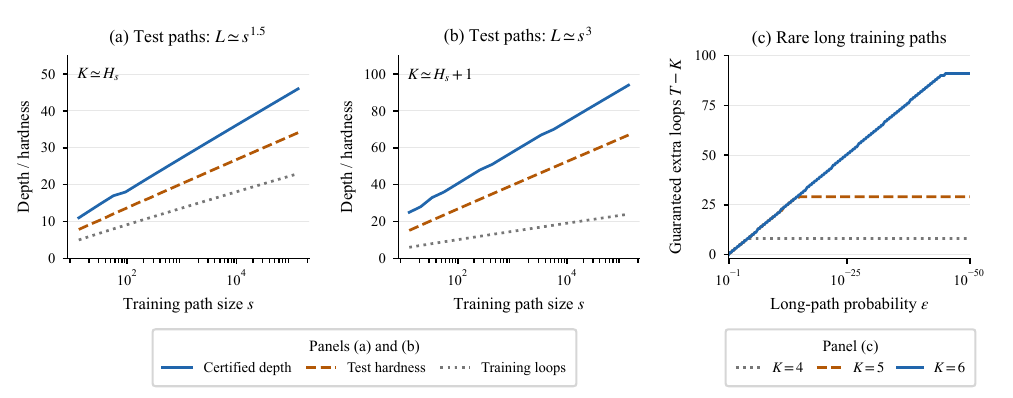}
\caption{Numerical illustrations of the extrapolation guarantees.
(a--b) The certified depth exceeds test-graph hardness for paths
with $L\simeq s^{1.5}$ and $L\simeq s^3$.
(c) Rarer long training paths permit more certified test loops,
until the bound saturates. The left legend applies to (a--b);
the right legend applies to (c).}
\label{fig:hardness-illustrations}
\end{figure}

For panels (a--b), choose $s$ nearest to
$\mathsf{Hard}(P_s)=K-\Delta$ and set
$L=s\lceil s^{q-1}\rceil$, $N=2s\lceil s^{q-1}\rceil$.
The pairs $(\Delta,q)=(0,1.5)$ and $(1,3)$ use
$K=5,\ldots,23$ and $K=6,\ldots,24$, respectively.
Both training and testing use the packed-path construction, with
component sizes $s$ and $L$, respectively. We compute the exact spectral
residual and the common depth in~\eqref{eq:uniform-upper}.
For panel (c), the training law mixes uniform relabelings of
$G_{128,4}$ and $G_{128,64}$ with probabilities $1-\varepsilon$ and
$\varepsilon$. We evaluate the tail-bound certificate from
\eqref{eq:delta-quantile} and~\eqref{eq:uniform-upper}, with
$H=\mathsf{Hard}(P_4)$, $K=4,5,6$, and
$\varepsilon\in[10^{-50},10^{-1}]$. All panels use the fixed experimental threshold.
\end{document}